\documentclass{article}
\usepackage{microtype}
\usepackage{graphicx}
\usepackage{subcaption}
\usepackage{booktabs} %

\usepackage{hyperref}

\usepackage[preprint]{icml2026}

\usepackage{amsmath}
\usepackage{amssymb}
\usepackage{mathtools}
\usepackage{amsthm}
\usepackage[utf8]{inputenc} %
\usepackage[T1]{fontenc}    %
\usepackage{hyperref}
\usepackage{url}
\usepackage{nicefrac}       %
\usepackage{xcolor}         %
\usepackage[svgnames]{xcolor}         %
\usepackage{multirow}
\usepackage{makecell}
\usepackage{adjustbox}
\usepackage{siunitx} 
\usepackage{colortbl}
\usepackage{bm}
\usepackage{xspace}
\usepackage[nameinlink]{cleveref}
\usepackage{enumitem}

\usepackage{amsfonts}       %
\usepackage{aliascnt}
\usepackage[flushleft]{threeparttable}
\usepackage[nameinlink]{cleveref}
\usepackage{soul}
\usepackage[most]{tcolorbox}

\renewcommand{\algorithmiccomment}[1]{\hfill $\triangleright$ #1}
\definecolor{OursColor}{HTML}{003442}
\definecolor{PAColor}{HTML}{F9CF80}
\definecolor{ReActColor}{HTML}{2E8CA6}

\usepackage{tikz}
\usetikzlibrary{positioning, arrows.meta, shapes, calc, fit, backgrounds, decorations.pathreplacing}

\definecolor{Red}{rgb}{0.768, 0.054, 0.054}
\definecolor{Blue}{rgb}{0.152, 0.294, 0.925}
\definecolor{Green}{rgb}{0,0.4,0.7}
\hypersetup{
    colorlinks=true,
    citecolor=teal,
    linkcolor=Red,
    urlcolor=Green,
}

\newtcolorbox[auto counter, number within=section]{promptbox}[2][]{
  colback=white,
  breakable,
  colframe=blue!80!black,
  coltitle=white,
  title=Prompt~\thetcbcounter: #2,
  fonttitle=\bfseries\normalsize,
  boxrule=1pt,
  arc=2mm,
  top=2mm,
  bottom=2mm,
  width=1.0\textwidth,
  nameref={Prompt},                 %
  #1 
}

\newtcolorbox[auto counter, number within=section]{examplebox}[2][]{
  colback=white,
  colframe=black,
  breakable,    
  boxrule=0.5pt,
  sharp corners,
  left=4pt,
  right=4pt,
  top=4pt,
  bottom=4pt,
  enhanced,
  title=Example~\thetcbcounter: #2,
  #1
}
\usepackage{pifont}

\newcommand{\cmark}{\ding{51}}

\usepackage{amsmath,amsfonts,bm}

\def\eqref#1{equation~\ref{#1}}

\def\1{\bm{1}}

\DeclareMathAlphabet{\mathsfit}{\encodingdefault}{\sfdefault}{m}{sl}
\SetMathAlphabet{\mathsfit}{bold}{\encodingdefault}{\sfdefault}{bx}{n}

\theoremstyle{plain}

\newaliascnt{proposition}{theorem}
\newtheorem{proposition}[proposition]{Proposition}
\aliascntresetthe{proposition}

\newaliascnt{lemma}{theorem}

\aliascntresetthe{lemma}

\newaliascnt{corollary}{theorem}

\aliascntresetthe{corollary}

\theoremstyle{definition}
\newaliascnt{definition}{theorem}
\newtheorem{definition}[definition]{Definition}
\aliascntresetthe{definition}

\newaliascnt{assumption}{theorem}
\newtheorem{assumption}[assumption]{Assumption}
\aliascntresetthe{assumption}

\theoremstyle{remark}
\newaliascnt{remark}{theorem}

\aliascntresetthe{remark}

\newcommand{\ours}{TAPE\xspace}

\newcommand{\eqautoref}[1]{\hyperref[#1]{Equation~(\ref*{#1})}}

\usepackage[textsize=tiny]{todonotes}

\icmltitlerunning{\ours: Tool-Guided Adaptive Planning and Constrained Execution in LM Agents}

\begin{document}

\twocolumn[
  \icmltitle{TAPE: Tool-Guided Adaptive Planning and Constrained Execution\\in Language Model Agents}

  \icmlsetsymbol{equal}{*}

  \begin{icmlauthorlist}
    \icmlauthor{Jongwon Jeong}{uwm}
    \icmlauthor{Jungtaek Kim}{uwm}
    \icmlauthor{Kangwook Lee}{uwm,krafton,ludo}
  \end{icmlauthorlist}

  \icmlaffiliation{uwm}{Electrical and Computer Engineering, University of Wisconsin--Madison}
  \icmlaffiliation{krafton}{KRAFTON}
  \icmlaffiliation{ludo}{Ludo Robotics}

  \icmlcorrespondingauthor{Kangwook Lee}{kangwooklee@krafton.com}

  \icmlkeywords{Tool-Guided Adaptive Planning, Tool-Guided Constrained Execution, Language Model Agents}
  \vskip 0.3in
]

\printAffiliationsAndNotice{}  %

\begin{abstract}
Language Model (LM) agents have demonstrated remarkable capabilities in solving tasks that require multiple interactions with the environment. 
However, they remain vulnerable in environments where a single error often leads to irrecoverable failure, particularly under strict feasibility constraints.
We systematically analyze existing agent frameworks, identifying imperfect planning and stochastic execution as the primary causes.
To address these challenges, we propose \textbf{T}ool-guided \textbf{A}daptive \textbf{P}lanning with constrained \textbf{E}xecution (\textbf{TAPE}).
TAPE enhances planning capability by aggregating multiple plans into a graph and employing an external solver to identify a feasible path.
During execution, TAPE employs constrained decoding to reduce sampling noise, while adaptively re-planning whenever environmental feedback deviates from the intended state.
Experiments across Sokoban, ALFWorld, MuSiQue, and GSM8K-Hard demonstrate that TAPE consistently outperforms existing frameworks, with particularly large gains on hard settings, improving success rates by 21.0 percentage points on hard settings on average, and by 20.0 percentage points for weaker base models on average. Code and data available at \href{https://github.com/UW-Madison-Lee-Lab/TAPE}{here}.
\vspace{-10pt}
\end{abstract}

\section{Introduction}

Language Models (LMs) have evolved into agents capable of understanding and interacting with external environments such as computers~\citep{LLMAgentSurvey}, simulation platforms~\citep{Smallville}, and physical robot environments~\citep{RobotLLMSurvey}.
In these applications, tools serve as interfaces that enable LM agents to reason, act, and observe within these environments~\citep{LLMAgentParadigmSurvey}.
For instance, search tools retrieve information from databases~\citep{SearchR1}, mouse and keyboard tools interact with computer interfaces~\citep{OSWorldAgent}, coding tools execute programs within computers~\citep{PAL}, and robotic tools perform actions in the physical world~\citep{CoPAL}.
By generating executable prompts for tool usage, LM agents can iteratively think and act while receiving observations from the environment, referred to as \emph{ReAct} frameworks~\citep{ReAct, Reflexion, RAP, ToT, ToolChain, WKM, ReflAct}.
Through this iterative process, they can solve complex tasks that require planning and adaptation~\citep{TALM, PAL}.

\begin{figure*}[t]
\centering
\includegraphics[width=.95\linewidth]{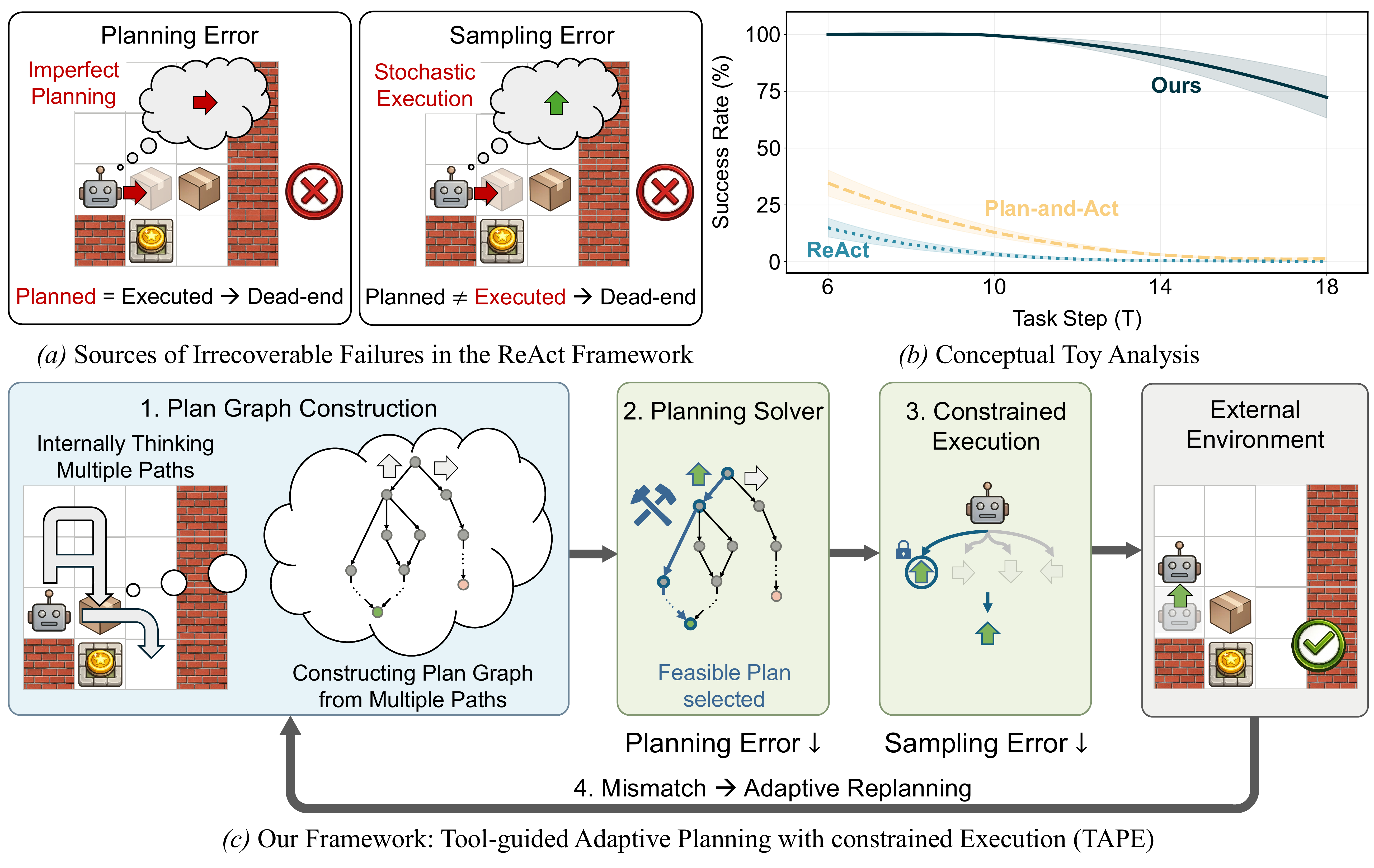}
\vspace{-8pt}
\caption{\textbf{Overview.} We illustrate our work using Sokoban, where the goal is to push all boxes onto target locations.
\textbf{(a) Sources of Irrecoverable Failure in the ReAct Framework.} A planning error occurs when the internal reasoning suggests a non-viable action (e.g., pushing a box against a wall); this makes the goal unachievable as the agent cannot pull the box from the wall, while a sampling error arises when LM stochasticity leads to an action deviating from the plan.
\textbf{(b) Conceptual Toy Analysis.} We model simplified agents by injecting planning and sampling errors into a feasible policy for Sokoban. We measure success rates as the task step $T$ increases, observing that existing frameworks degrade rapidly as $T$ grows. See \autoref{app:concept_proof_details} for details.
\textbf{(c) Our Framework.} \ours generates and aggregates multiple plans into a graph and uses a solver to select a feasible path, thereby reducing planning errors. Then, it enforces constrained execution to suppress sampling errors.}
\label{fig:motivation}
\vspace{-16pt}
\end{figure*}

Although the ReAct frameworks have shown impressive capabilities in various domains~\citep{AgentBench}, these frameworks are highly ineffective when even a few mistakes can cause irrecoverable harms, making it infeasible to achieve the goal thereafter.
In practice, feasibility constraints, such as time and cost budgets, limits on tool usage, and strict safety requirements, are one common source of irrecoverable failures.
For instance, coding agents often operate under latency or API-cost budgets~\citep{LLMCompiler, BudgetTool, BATS}, robotic agents must satisfy strict safety constraints~\citep{PlugSafety, CaStL}, and life-simulation agents face API-usage limits~\citep{Smallville, InfectedSmallVille}.
Under such constraints, incorrect tool use can exhaust the remaining budget or violate constraints, leaving no feasible sequence of tools that reaches the goal.
This motivates our research question:
\begin{tcolorbox}[colback=yellow!20,colframe=yellow!50!black,
                  boxrule=0.8pt,arc=3pt]
\emph{How can LM agents maximize the success rate on tasks in the presence of irrecoverable failures?}
\end{tcolorbox}

To address this question, we first analyze ReAct frameworks and identify two distinct sources of irrecoverable failures: \emph{planning error}~\citep{LLMPlanningCapability} and \emph{sampling error}~\citep{ETAL}.
A planning error occurs when an agent's internal planning (i.e., reasoning) is imperfect, leading it to recommend a non-viable action (see \autoref{fig:motivation}a, left).
A sampling error can occur even when the agent's internal reasoning is correct, because stochastic token generation may generate an action different from the planned one (see \autoref{fig:motivation}a, right).
By simulating a simplified agent where we inject these errors into a feasible policy, we find that both failures compound as the task horizon increases, reducing the overall success rate (see \autoref{fig:motivation}b).
Similar issues persist for Plan-and-Act (PA) frameworks~\citep{PlanAndSolve, PlanAndAct, AdaPlanner}.
While generating a full plan before execution mitigates sampling errors, PA remains brittle to planning errors, resulting in suboptimal success rates in our analysis (see \autoref{fig:motivation}b). Refer to \autoref{sec:theoretical_analysis} for the theoretical validation of this analysis.

To mitigate failures from both error sources, we propose an external \textbf{T}ool-guided \textbf{A}daptive \textbf{P}lanning with constrained \textbf{E}xecution framework (\textbf{\ours}), as shown in \autoref{fig:motivation}c. 
Specifically, our framework generates multiple candidate plans and aggregates them into a plan graph, then uses an external solver (e.g., Integer Linear Programming (ILP)~\citep{LP}) to select a feasible path, mitigating failures due to planning errors by optimally selecting among diverse candidates. 
Next, our method enforces constrained execution by constrained decoding~\citep{GuidedGeneration} to the selected next action, which suppresses sampling errors. 
At each step, if \ours detects a mismatch between the predicted and realized observations, it updates the plan graph and re-selects a feasible plan. As shown in \autoref{fig:motivation}b, our framework achieves a higher success rate than other frameworks.

We evaluate the proposed framework on benchmarks built from Sokoban~\citep{Sokoban}, ALFWorld~\citep{ALFWorld}, MuSiQue~\citep{MuSiQue} and GSM8K-Hard~\citep{PAL} by adding feasibility constraints, e.g., budgets or tool/action limits, that make mistakes difficult to recover from.
Experimental results show that our framework consistently outperforms the ReAct frameworks across all benchmarks. The performance gap is most significant in hard tasks and for weaker base models, proving that our framework effectively mitigates the irrecoverable failures.
Finally, our ablation study confirms that all proposed components are essential for achieving success in environments where failures are irrecoverable.

To sum up, our contributions are as follows:
\begin{itemize}[itemsep=0.7mm, parsep=1pt, leftmargin=*, topsep=0pt, partopsep=0pt]
\item We formalize planning and sampling errors as two sources of irrecoverable failures in the ReAct frameworks, and theoretically characterize their impact on success probability under feasibility constraints.
\item We propose an external tool-guided adaptive planning with constrained execution framework. \ours mitigates planning errors via solver-based path selection over a plan graph with replanning, and reduces sampling errors via constrained decoding of planned actions.
\item We construct constrained variants of existing agent benchmarks and show consistent success-rate improvements over other frameworks.
\end{itemize}

\section{Problem Formulation} \label{sec:problem_formulation}

\subsection{Agentic Task}
We consider an \emph{agentic task} in which an LM agent iteratively interacts with an external environment to solve a given problem.
We formalize this process as a goal-conditioned Markov Decision Process (G-MDP)~\citep{nasiriany2019planning},
denoted by $\mathcal{M} = (\mathcal{G}, \mathcal{A}, \mathcal{S}, P, R)$.
Here, $\mathcal{G}$ is the goal space,
$\mathcal{A}$ is the action space,
$\mathcal{S}$ is the state space,
$P$ denotes the transition dynamics,
and $R: \mathcal{G} \times \mathcal{S} \to \{0, 1\}$ is a goal-dependent reward function, where we write $R_g(s)$.
At the beginning of an episode, a goal $g \in \mathcal{G}$ is given.
When actions incur costs and the agent operates under budgets, we define a cost function $C: \mathcal{S} \times \mathcal{A} \to \mathbb{R}^k$ and a budget vector $B \in \mathbb{R}^k$.
The state $s_t \in \mathcal{S}$ is defined as the interaction history up to timestep $t$,
i.e., $s_t = (a_0, o_0, a_1, o_1, \dots, a_{t-1}, o_{t-1})$ with $s_0 = \emptyset$.
At each timestep $t \in \{0,1,\dots\}$,
the agent selects an action $a_t \in \mathcal{A}$ conditioned on $g$ and $s_t$.
The environment responds with an observation $o_t = (o_t^{\mathrm{status}}, o_t^{\mathrm{return}})$,
where $o_t^{\mathrm{status}}$ indicates the execution status (e.g., success or failure),
and $o_t^{\mathrm{return}}$ denotes the returned value (e.g., tool outputs or error messages).
If the task considers budget constraints, the observation includes the remaining-budget component, i.e., $o_t=(o_t^{\mathrm{status}},o_t^{\mathrm{return}},o_t^{\mathrm{budget}})$.
Given $s_t$ and $a_t$,
the next state $s_{t+1}$ is realized according to $P(\cdot \mid s_t, a_t)$,
which corresponds to appending $(a_t, o_t)$ to $s_t$.
An episode $\tau=(s_0,a_0,s_1,\dots,s_T)$ is successful if it terminates at a state $s_T$ such that $R_g(s_T)=1$ and, when budget constraints are present, $\sum_{t=0}^{T-1} C(s_t,a_t) \preceq B$, where $\preceq$ denotes element-wise inequality.
Such budgets can represent time limits, cost limits, limits on the number of tool calls, or so on~\citep{BudgetTool, BATS, ma2026timely}.
If the agent enters a dead-end state (e.g., by violating the budget constraint), from which reaching the goal is impossible regardless of subsequent actions and transitions, we terminate the episode at $s_T$ and set $R_g(s_T)=0$.

\subsection{Language Model Agent Framework}\label{sec:language_model_agents}

An LM agent is defined as a stochastic policy $\pi_\theta(a_t \mid s_t, g)$ that selects an action given the current environment state $s_t$ and goal $g$. Since the agent cannot observe the concrete return values until it actually interacts with the environment, it must rely on an internal world model to simulate and evaluate potential trajectories beforehand~\citep{RAP}.
To facilitate this internal planning, the agent operates on an abstract state $\hat{s}_t \coloneqq z(s_t) \coloneqq (a_0, o_0^{\mathrm{status}}, \dots, a_{t-1}, o_{t-1}^{\mathrm{status}})$, which distills the raw history into essential execution statuses. In a constrained G-MDP, this representation is extended to $\hat{s}_t^{c} \coloneqq z^{c}(s_t) \coloneqq (a_0, (o_0^{\mathrm{status}}, o_0^{\mathrm{budget}}), \dots, a_{t-1}, (o_{t-1}^{\mathrm{status}}, o_{t-1}^{\mathrm{budget}}))$ to track resource consumption. We assume the agent has an internal world model $P_\theta(\hat{s}_{t+1} \mid \hat{s}_t, a_t)$, $R_{g,\theta}(\hat{s}_T)$, and $C_\theta(\hat{s}_t^{c}, a_t)$ that approximate the environment's true transition dynamics $P$, abstract reward $R_g$, and cost function $C$.
Using the internal world model, the LM agent can perform internal planning to identify an action path that maximizes rewards.

\paragraph{ReAct Framework.}
Following ReAct~\citep{ReAct}, we decompose the agent's decision-making into two stages, ``Think'' and ``Act''.
At each step $t$, ``Think'' contains various types of textual reasoning~\citep{ReAct, Reflexion, ReflAct} or advanced planning strategy~\citep{RAP, RAFA, ToT, ToolChain, WKM, DGAP, ThoughtOfSearch}.
We interpret this ``Think'' as \emph{planning} over an internal abstract world model~\citep{RAP}.
Starting from the current interaction state, the agent predicts future abstract states over a lookahead horizon and outputs a planned next action $\hat a_t \in \mathcal{A}$~\citep{RAFA}.
``Act'' then samples an executed action conditioned on the state and the planned action.
Accordingly, this decision process is formalized as
\begin{equation}
    \hat{a}_t \sim \pi^{\mathrm{Think}}_\theta(\cdot \mid s_t, g),
\qquad
a_t \sim \pi^{\mathrm{Act}}_\theta(\cdot \mid s_t,\hat{a}_t, g).
\end{equation}

\paragraph{Plan-and-Act Framework.}
Plan-and-Act (PA) framework leverages an internal abstract world model to perform full planning before step-wise execution~\citep{PlanAndSolve, ReWoo, MPO, PlanAndAct}.
Given a goal $g$, the agent explores the abstract state space and produces an abstract plan
$\hat{\tau}=(\hat{s}_0^{\hat \tau},\bar{a}_0,\hat{s}_1^{\hat \tau},\bar{a}_1,\dots,\hat{s}_T^{\hat \tau})$,
where the planned action at step $t$ is induced by the internal world model $(P_\theta,R_{g,\theta},C_\theta)$ through planning, i.e.,
$\bar a_t \sim \pi^{\mathrm{Think}}_\theta(\cdot\mid \hat s_t^{\hat \tau}, g)$.
Once the plan $\hat\tau$ is provided in-context, the agent conditions its step-wise decision on the current interaction state $s_t$ and the plan guidance.
We model PA's in-context use of the provided plan as follows.
When the plan is applicable at step $t$, meaning that the current abstract state matches the abstract state, $z(s_t)=\hat s_t^{\hat \tau}$, the PA agent directly follows the planned action with probability $p\in[0,1]$.
Otherwise, it falls back to the ReAct-style decision process.
Formally, if $z(s_t)=\hat s_t^{\hat \tau}$,
\begin{equation}\label{eq:plan_and_act}
a_t =
\begin{cases}
\bar a_t, & \text{with probability} \ \ p_{\mathrm{follow}},\\
a_t^{\mathrm{ReAct}}, & \text{with probability} \ \ 1-p_{\mathrm{follow}},
\end{cases}
\end{equation}
and if $z(s_t)\neq\hat s_t^{\hat \tau}$, $a_t=a_t^{\mathrm{ReAct}}$, where $\tilde a_t \sim \pi^{\mathrm{Think}}_\theta(\cdot\mid s_t, g),a_t^{\mathrm{ReAct}} \sim \pi^{\mathrm{Act}}_\theta(\cdot\mid s_t,\tilde a_t, g).$
Here, $\tilde a_t$ denotes the planned next action produced by ReAct.

\subsection{Errors from Two Different Sources}
\label{subsec:error_sources}

LM agents are prone to fail agentic tasks due to internal \emph{planning error} and \emph{sampling error} in LMs.
In practice, planning error arises from limited planning capability or mismatch between the agent's internal world model and the true world model~\citep{LLMPlanningCapability}.
The sampling error can arise from stochastic token generation, prompt sensitivity, and formatting drift~\citep{ETAL}.
To quantify both errors, in G-MDP, we define viable actions as in \autoref{def:viable_action_set_main}.

\begin{definition}
\label{def:viable_action_set_main}
Define the viable action set as
\begin{align}
&\mathcal{A}_{\mathrm{viable}}(s_t) \nonumber\\
&\coloneqq\!\Big\{a\in \mathcal{A} \big| \exists K \ge 1, \Pr\!\big(R_g(s_{t+K}) = 1 \!\!\mid\!\! s_t, a\big) > 0
\Big\}.
\end{align}
\end{definition}

Using the viable action set, we formalize planning and sampling errors as follows. For simplicity, we assume a constant per-step error rate that is time-invariant.

\begin{assumption}[Planning Error]
\label{ass:planning_error}
Let $\hat a_t$ denote the agent's planned action at step $t$.
For simplicity, we model the planning error as the constant
\begin{equation}
\label{eq:delta_def_main}
\epsilon_{\mathrm{p}} \coloneqq \Pr\left(\hat a_t \notin \mathcal{A}_{\mathrm{viable}}(s_t)\right).
\end{equation}
\end{assumption}

\begin{assumption}[Sampling Error]
\label{ass:sampling_error}
Let $\hat a_t$ denote the agent's planned action at step $t$.
For simplicity, we model sampling error by the following constant rates.
\begin{equation}
\epsilon_{\mathrm{s}} \coloneqq \Pr(a_t \neq \hat a_t),
\end{equation}
Moreover, we assume that $\epsilon_{\mathrm{s}}$ does not depend on whether the planned action is viable, i.e.,
$\Pr(a_t \neq \hat a_t \mid \hat a_t \in \mathcal{A}_{\mathrm{viable}}(s_t))
=
\Pr(a_t \neq \hat a_t \mid \hat a_t \notin \mathcal{A}_{\mathrm{viable}}(s_t))
=
\epsilon_{\mathrm{s}}$.
\end{assumption}

We further characterize the consequences of sampling error as $\delta_{\mathrm{b}}$ and $\delta_{\mathrm{r}}$ in \autoref{ass:sampling_error_consequences}. $\delta_{\mathrm{b}}$ quantifies how often the sampling error breaks viability when the planned action is viable, while $\delta_{\mathrm{r}}$ captures how often the sampling error recovers viability when the planned action is non-viable.

\begin{assumption}
\label{ass:sampling_error_consequences}
When the planned action $\hat{a}$ is viable, $\delta_{\mathrm{b}}$ denotes the probability that deviating from it breaks viability as $\delta_{\mathrm{b}}
\coloneqq
\Pr\!\big(a_t \notin \mathcal{A}_{\mathrm{viable}}(s_t)\ \big|\ 
\hat a_t \in \mathcal{A}_{\mathrm{viable}}(s_t),\ a_t \neq \hat a_t\big)$.
On the other hand, when the planned action is non-viable, $\delta_{\mathrm{r}}$ denotes the probability that deviating from it recovers viability as $\delta_{\mathrm{r}}
\!\coloneqq\!
\Pr\!\big(a_t \!\in\! \mathcal{A}_{\mathrm{viable}}(s_t)\ \big|\ 
\hat{a}_t \!\notin\! \mathcal{A}_{\mathrm{viable}}(s_t),\ a_t \!\neq\! \hat{a}_t\big)$. 
\end{assumption}

\begin{figure*}[t]
    \centering
    \includegraphics[width=1.0\linewidth]{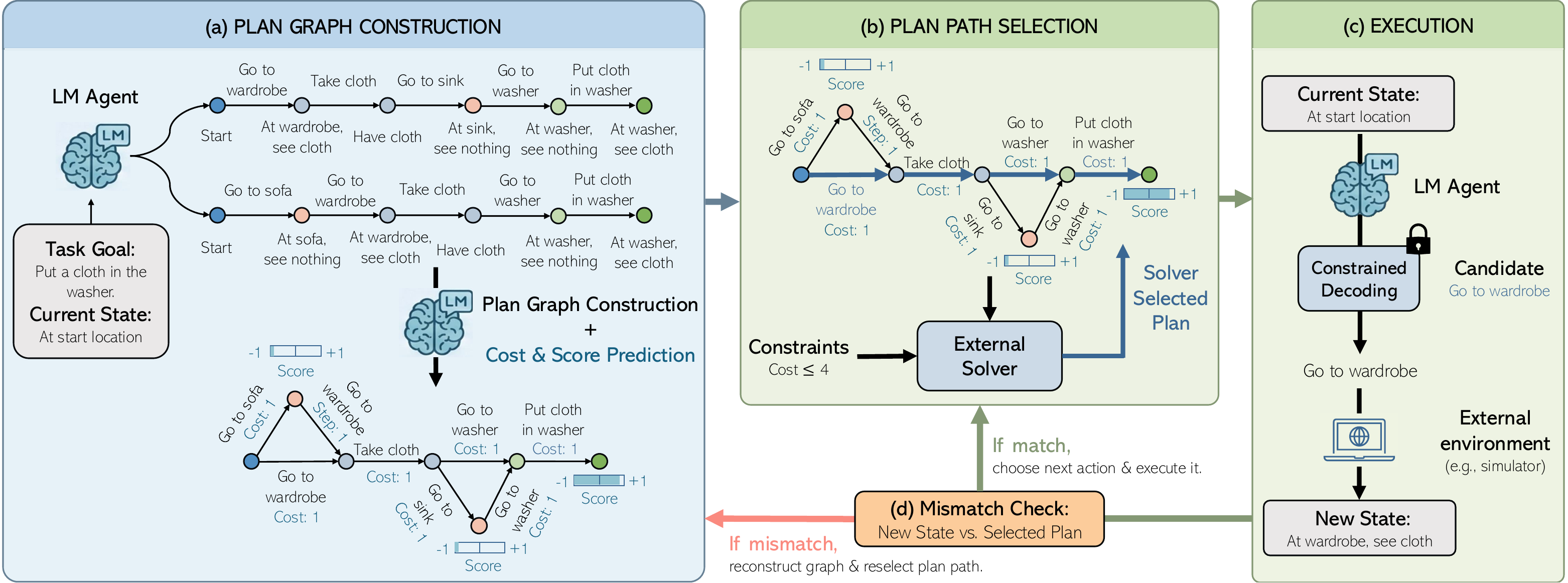}
\vspace{-15pt}
\caption{
    \textbf{Overview of \ours.}
    The proposed framework consists of four steps.
    (a) \textbf{Plan Graph Construction}: The LM samples multiple trajectories, which are aggregated into a plan graph with predicted costs and scores.
    (b) \textbf{Plan Path Selection}: An external solver (e.g., ILP) identifies the optimal path (blue arrows) subject to constraints.
    (c) \textbf{Constrained Execution}: The agent executes the selected actions using constrained decoding to eliminate sampling errors.
    (d) \textbf{Mismatch Check}: If a mismatch occurs between the predicted and realized state, the agent re-performs plan graph construction and path selection; otherwise, the agent executes the next planned action.
}
\label{fig:method_overview}
\vspace{-16pt}
\end{figure*}

\section{Our Approach: \ours}

In this section, we introduce \textbf{\ours} (\textbf{T}ool-guided \textbf{A}daptive \textbf{P}lanning with constrained \textbf{E}xecution), a framework designed to reduce both planning and sampling errors. First, our framework performs (i) \emph{plan graph construction} by recombining multiple abstract trajectories generated by an LM, which increases the probability that a feasible plan exists within the planning space (\autoref{sec:plan_graph_construction}). Next, \ours operates (ii) \emph{plan path selection} on this graph using an external solver, such as Integer Linear Programming (ILP)~\citep{LP}, based on predicted costs and rewards (\autoref{sec:path_selection}).
Given the constructed plan graph and its predicted costs and rewards, the solver returns an optimal feasible path within the graph when one exists, effectively reducing planning errors.
Then, our framework performs (iii) \emph{constrained execution} by restricting the action space at decoding time~\citep{GuidedGeneration} so that only the next prescribed action on the selected path is admissible, thereby suppressing sampling errors (\autoref{sec:constrained_execution}). Lastly, to remain robust against mismatches between realized observations and the internal plan, the framework adaptively re-performs (iv) \emph{plan graph construction and path selection} using newly observed information whenever a discrepancy arises (\autoref{sec:mismatch_replanning}). A conceptual overview on ALFWorld~\citep{ALFWorld} is shown in \autoref{fig:method_overview}.

\subsection{Plan Graph Construction}\label{sec:plan_graph_construction}

\paragraph{Graph Construction.} To construct a plan graph $\mathcal G=(V,E)$, where $v\in V$ is a state node and $e\in E$ is an edge representing an action $\bar a$, \ours first generates $M$ abstract plans $\hat\tau^{(m)}=\left(\hat s^{(m)}_0,\bar a^{(m)}_0,\hat s^{(m)}_1,\bar a^{(m)}_1,\dots,\hat s^{(m)}_{L_m}\right)$, where $m=1,\dots,M$, $\hat s^{(m)}_\ell\in\hat{\mathcal S}$, and $\bar a^{(m)}_\ell\in\mathcal A$.
Then, our framework folds these sampled paths by merging states that share the same core information.
Specifically, \ours merges multiple abstract states $\hat{s}$ into a single node $v \in V$ if they represent the same observation and task progress.
For example, in ALFWorld, many distinct states (e.g., different intermediate tool outputs or error messages) correspond to the same node representing the agent's current location, the objects in the inventory, and the task progress (such as whether a required object has been found, cleaned, or placed).
We denote this merging function as $f_\theta: \hat{\mathcal{S}} \to \mathcal{V}$.
For the merged nodes, \ours reassigns edges as $e_{ij}=\big(v_i \xrightarrow{\ \bar a^{(m)}_\ell} v_j\big)\in E$,
where $v_i = f_\theta(\hat s^{(m)}_\ell)$ and $v_j = f_\theta(\hat s^{(m)}_{\ell+1})$ for each consecutive pair in the sampled plans.

\paragraph{Score and Cost Prediction.}
Using the internal abstract world model $(P_\theta, R_{g,\theta}, C_\theta)$ in LMs, \ours assigns a scalar reward to each node $u \in V$ and a predicted cost to each edge $e \in E$.
Specifically, for terminal nodes $v_{\mathrm{ter}}$, our framework estimates the expected reward as $\hat{r}_\theta(u) \coloneqq \mathbb{E}_\theta [R_{g,\theta}(\hat{s}) \mid f_\theta(\hat{s}) = v_{\mathrm{ter}}]$, representing the expected reward of being in $v_{\mathrm{ter}}$.
For each edge $e = (v \xrightarrow{a} v') \in E$, \ours predicts the cost vector as $\hat{\mathbf{c}}_\theta(e) \coloneqq \mathbb{E}_\theta [\mathbf C_\theta(\hat s, a) \mid f_\theta(\hat s) = v]$, which estimates the expected budget consumption of executing the edge-labeled action at node $v$.
In constrained settings, these predicted costs $\{\hat{\mathbf{c}}_\theta(e)\}_{e \in E}$ are used to enforce feasibility under the remaining budget $\mathbf{b}_t$.

\subsection{Plan Path Selection}
\label{sec:path_selection}
Given the graph $\mathcal G=(V,E)$, the current node $v_0$, and the set of terminal nodes $V_{\mathrm{ter}}\subseteq V$, our framework selects a single directed walk on $\mathcal G$ by solving an optimization problem via an external solver (e.g., Integer Linear Programming (ILP)~\citep{LP}).
To obtain a tractable finite-size ILP, the solver optimizes over a finite horizon $L_{\max} \ge 1$ using a time-expanded formulation. For each edge $e=(v\xrightarrow{\ a\ }v')\in E$ and step $\ell\in\{0,1,\dots,L_{\max}-1\}$, let $x_{e,\ell}\in\{0,1\}$ indicate whether the walk takes edge $e$ at step $\ell$. $\mathrm{src}(e)=v$ and $\mathrm{tgt}(e)=v'$ are the source and target of $e$. In the G-MDP setting, the ILP formulation is as follows:
\begin{align}
\max_{x}\quad
& \sum_{\ell=0}^{L_{\max}-1}\ \sum_{e\in E} \hat r_\theta\big(\mathrm{tgt}(e)\big)\,x_{e,\ell} \label{eq:obj} \\
\text{s.t.}\quad & \,\,\,\, \sum_{e\in E} x_{e,\ell}=1, \quad \forall \ell=0,\dots,L_{\max}-1, \label{eq:single_action} \\
& \,\,\,\,\sum_{\mathclap{e\in E:\ \mathrm{src}(e)=v_0}} x_{e,0}=1, \label{eq:start_node} \\
& \,\,\,\,\sum_{\mathclap{e\in E:\ \mathrm{tgt}(e)\in V_{\mathrm{ter}}}} x_{e,L_{\max}-1}=1,\label{eq:goal_node} \\
& \,\,\,\,\begin{aligned}
&\sum_{\mathclap{e\in E: \mathrm{tgt}(e)=v}} x_{e,\ell-1} = \sum_{\mathclap{e\in E:\mathrm{src}(e)=v}} x_{e,\ell}, \\
& \quad\quad \forall v \in V,\forall \ell \in \{1, \dots, L_{\max}-1\}
\label{eq:flow_conservation}\end{aligned}
\\
& \,\,\,\, x_{e,\ell}\in\{0,1\}, \quad \forall e\in E,\ \forall \ell. \label{eq:binary_var}
\end{align}

\eqautoref{eq:obj} maximizes the total accumulated reward of the visited nodes.
The constraints ensure the validity of the selected path.
Specifically, \eqautoref{eq:single_action} ensures that the agent takes exactly one action at each step.
\eqautoref{eq:start_node} and \eqautoref{eq:goal_node} specify that the path starts at $v_0$ and ends at one of the terminal nodes in $V_{\mathrm{ter}}$ at step $L_{\max}$.
Finally, \eqautoref{eq:flow_conservation} ensures that if the agent arrives at node $v$ at step $\ell-1$, it must depart from $v$ at step $\ell$, thereby forming a continuous walk.
If there exist budget constraints, we can formulate the optimization problem by adding the budget constraints, shown in \autoref{app:constrained_ilp}.

\subsection{Constrained Execution}\label{sec:constrained_execution}

Given the optimal path selected by the solver, denoted as $\pi^\star=(v^{\pi^\star}_0 \xrightarrow{\bar a^\star_0} v^{\pi^\star}_1 \xrightarrow{\bar a^\star_1}\cdots)$, \ours executes the plan by restricting the admissible action set to the prescribed action whenever the path is applicable.
Let $v_t$ be the current node in the environment at step $t$.
If the current node matches the planned node (i.e., $v_t = v^{\pi^\star}_t$), we enforce the LM $\pi_\theta^{\mathrm{Act}}$ to generate the planned action $\bar{a}^\star_t$ as $a_t$.
This is implemented via constrained decoding~\citep{GuidedGeneration}, which constrains the LM to follow $\bar a^\star_t$ by fixing the tool choice and enforcing the exact tool-call format.

\subsection{Mismatch Check and Replanning}
\label{sec:mismatch_replanning}

In practice, a mismatch between the predicted state in the plan and the realized state can invalidate the current path.
For example, the environment might transition to an unexpected state, or the remaining budget might evolve differently than estimated.
To address this, our framework verifies the consistency between the real state node $\hat s_{t+1}$ and the predicted state node $v_{t+1}^{\pi^\star}$.
Specifically, if \ours confirms that $v_{t+1}$ matches $v_{t+1}^{\pi^\star}$, \ours proceeds to execute the next planned action along $\pi^\star$.
Conversely, if our framework detects a mismatch, \ours regenerates multiple abstract plans, constructs a new plan graph with updated costs and rewards, and solves for a new optimal path via the external solver.

\section{Theoretical Analysis} \label{sec:theoretical_analysis}

In this section, we analyze the success probability of agents within the framework in \autoref{sec:language_model_agents}.
Based on the planning and execution errors in \autoref{subsec:error_sources}, we explain how these errors propagate and reduce the overall success probability.
We define the success probability as follows.

\begin{definition}[Success Probability]
\label{def:success_event}
Let $\tau=(s_0,a_0,s_1,\dots)$ denote a random trajectory generated by the agent, and let
$T_g \coloneqq \min\{t \ge 0 : R_g(s_t)=1\}$
be the first goal-reaching step, with the convention $\min\emptyset \coloneqq \infty$.
We define the success event as $\mathcal S \coloneqq \left\{ T_g < \infty \right\} \cap 
\bigcap_{t=0}^{T_g-1}\left\{a_t \in \mathcal{A}_{\mathrm{viable}}(s_t)\right\}$, and we refer to $\Pr(\mathcal S)$ as the success probability.
\end{definition}

\subsection{Existing Frameworks}
\label{sec:theoretical_analysis_existing}

From \autoref{sec:problem_formulation}, we derive the upper bound of the success probability of the ReAct framework as follows.

\begin{proposition}[]
\label{prop:react_success_rate}
Assume that any successful trajectory must take at least $T$ action selections.
Then, the ReAct success probability is bounded by $U_{\mathrm{ReAct}}$:
\begin{equation}
\label{eq:react_success_upper}
\Pr(\mathcal{S})
\le
U_{\mathrm{ReAct}} \coloneqq
\Big((1-\epsilon_{\mathrm{p}})(1-\epsilon_{\mathrm{s}}\delta_{\mathrm{b}}) + \epsilon_{\mathrm{p}}\epsilon_{\mathrm{s}}\delta_{\mathrm{r}}\Big)^{T},
\end{equation}
where equality holds when the task exactly ends at $T$. If $(1-\epsilon_{\mathrm{p}})\delta_{\mathrm{b}} \ge \epsilon_{\mathrm{p}}\delta_{\mathrm{r}}$, $U_{\mathrm{ReAct}}$ increases as $\epsilon_{\mathrm{p}}$ and $\epsilon_{\mathrm{s}}$ decrease.
\end{proposition}

\autoref{prop:react_success_rate} implies that as $T$ grows, per-step errors compound and sharply reduce the overall success rate.
Also, if $(1-\epsilon_{\mathrm{p}})\delta_{\mathrm{b}} \ge \epsilon_{\mathrm{p}}\delta_{\mathrm{r}}$, the success probability can be increasing by reducing planning and sampling error. $(1-\epsilon_{\mathrm{p}})\delta_{\mathrm{b}} \ge \epsilon_{\mathrm{p}}\delta_{\mathrm{r}}$ is reasonable since deviations from the planned one are more likely to break viability than to recover it in agents.

Plan-and-Act (PA) framework utilizes a pre-generated plan as in-context guidance. This guidance effectively reduces execution stochasticity, lowering the sampling error. This comparison is formalized in~\autoref{prop:pa_vs_react}.

\begin{proposition}
\label{prop:pa_vs_react}
Let $\alpha \coloneqq \Pr(z(s_t) = \hat s_t^{\hat\tau})$ denote the probability that the plan is aligned at each step.
Under the same assumptions in \autoref{prop:react_success_rate}, we have the upper bound of success probability for PA as $U_{\mathrm{PA}} \coloneqq \left((1-\epsilon_{\mathrm{p}})\big(1-\epsilon_{\mathrm{s,PA}}\delta_{\mathrm{b}}\big) + \epsilon_{\mathrm{p}} \epsilon_{\mathrm{s,PA}}\delta_{\mathrm{r}}\right)^{T}$,
where $\epsilon_{\mathrm{s_\mathrm{PA}}} = (1-\alpha p_{\mathrm{follow}})\epsilon_{\mathrm{s}}$.
Consequently, we have the following:
\begin{equation}
U_{\mathrm{PA}} \ge U_{\mathrm{ReAct}},
\end{equation}
where equality holds when $\alpha=0$ or $p_{\mathrm{follow}}=0$.
\end{proposition}

\autoref{prop:pa_vs_react} suggests that the upper bound of success probability of PA is higher than that of ReAct due to reducing the sampling errors from $\epsilon_{\mathrm{s}}$ to $\epsilon_{\mathrm{s}_\mathrm{PA}}$.
However, note that the planning error $\epsilon_{\mathrm{p}}$ remains.

\subsection{Our Framework}
\label{sec:theoretical_analysis_existing_ours}

\ours aggregates $N$ sampled plans into a plan graph. This structure increases the diversity of action candidates at each step.
Let $v_t$ denote the planning node in the plan graph corresponding to the current state $s_t$, i.e., $v_t=f_\theta(s_t)$.
We define the set of candidate actions at $v_t$ as $\hat{\mathcal{A}}(v_t)\coloneqq \{\bar a\in\mathcal A \mid \exists\, (v_t\xrightarrow{\ \bar a\ }u_{t+1})\in E\}$, and let $d(v_t) \coloneqq |\hat{\mathcal{A}}(v_t)|$ be the number of distinct actions proposed by the aggregated plans.
Note that $d(v_t) \ge 1$ for any non-terminal node $v_t$.

\begin{proposition}
\label{prop:ours_vs_baselines}
Assume that a task requires $T$ steps, the external solver selects a viable action whenever one exists, and constrained decoding eliminates sampling error (i.e., $\epsilon_{\mathrm{s}} \approx 0$). Then, the upper bound of the success probability for \ours is given by $U_{\mathrm{ours}} \coloneqq \prod_{t=0}^{T-1} \left( 1 - (\epsilon_{\mathrm{p}})^{d(v_t)}\right)$. Consequently, we have
\begin{equation}
U_{\mathrm{ours}} \ge U_{\mathrm{PA}} \ge U_{\mathrm{ReAct}},
\end{equation}
where the equality between $U_{\mathrm{ours}}$ and $U_{\mathrm{PA}}$ holds if and only if $d(v_t)=1$ and $\epsilon_{\mathrm{s,PA}}\delta_{\mathrm{b}} = \epsilon_{\mathrm{s,PA}}\delta_{\mathrm{r}} = 0$.
\end{proposition}

\autoref{prop:ours_vs_baselines} confirms that \ours theoretically guarantees the highest success probability among the compared frameworks by exponentially reducing planning errors from $\epsilon_{\mathrm{p}}$ to $(\epsilon_{\mathrm{p}})^{d(v_t)}$ via selecting a viable action from multiple candidates, while also eliminating sampling errors $\epsilon_{\mathrm{s}} \approx 0$ via constrained decoding.

\begin{figure*}[t]
    \centering
    \includegraphics[width=\linewidth]{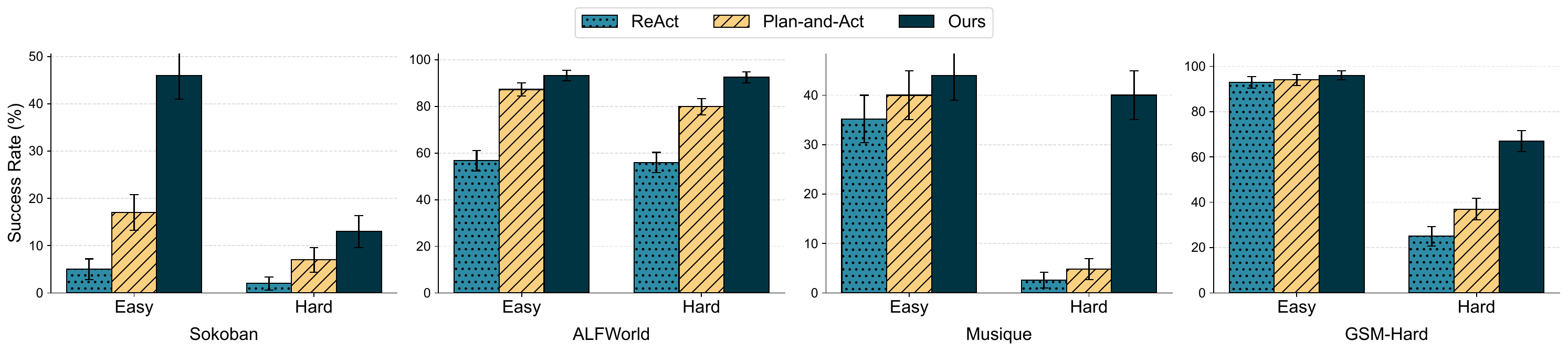}
\vspace{-15pt}
\caption{\textbf{Success rates across four agentic tasks.} We evaluate our framework against ReAct and Plan-and-Act on Sokoban, ALFWorld, Musique, and GSM-Hard. We use \texttt{gpt-4.1-mini} for LM backbone. We find that \ours consistently demonstrates superior performance over existing frameworks in both easy and hard settings.}
    \label{fig:main_result}
\vspace{-10pt}
\end{figure*}

\section{Empirical Analysis}
\label{sec:experiment}

We compare \ours against two representative agent frameworks: (i) \textbf{ReAct} \citep{ReAct} which interleaves reasoning and acting, and (ii) \textbf{Plan-and-Act (PA)} \citep{PlanAndAct} which executes a pre-generated plan. We note that various advanced prompting or reflection techniques can be orthogonally integrated into these frameworks; in this paper, we adapt \citet{MPO} for the implementation. Therefore, we focus on evaluating the fundamental structural differences between the frameworks.

For the evaluation, we consider four benchmarks characterized by frequent irrecoverable states: Sokoban, ALFWorld, GSM8K-Hard, and MuSiQue.
\textbf{Sokoban}~\citep{Sokoban} is a classic planning puzzle requiring the agent to push boxes to target locations without creating deadlocks.
\textbf{ALFWorld}~\citep{ALFWorld} involves embodied decision-making in a simulated household environment.
We also include \textbf{MuSiQue}~\citep{MuSiQue} for multi-hop factual reasoning using retrieval tools and \textbf{GSM8K-Hard}~\citep{PAL} for mathematical reasoning using arithmetic tools.
We report the success rate for each task. For more details about the experimental setting, see \autoref{app:experimental_details}.

\subsection{Overall Results}
The overall results on agentic benchmarks using \texttt{gpt-4.1-mini}~\citep{openai2025gpt41} are summarized in \autoref{fig:main_result}.
We compare \ours with ReAct and PA across four benchmarks characterized by frequent irrecoverable states, and find that \ours consistently outperforms the baselines.
In particular, \ours improves over ReAct, suggesting that our method mitigates both planning and sampling errors, whereas ReAct suffers from their accumulation.
Additionally, PA frequently outperforms ReAct, which may indicate that providing an explicit plan as in-context guidance helps reduce sampling errors.
Furthermore, \ours consistently surpasses both ReAct and PA.
As shown in \autoref{app:additional_experiment}, we additionally compare against Best-of-$N$~\citep{BestOfN, AgentDistill} of ReAct and PA, which sample the same number of plans at each step as \ours, and observe that \ours still consistently outperforms them.
Overall, these results suggest that \textbf{\ours successfully minimizes the planning and sampling errors inherent in existing frameworks}.
Notably, we observe substantial performance gains even in scenarios where existing frameworks achieve near-zero success rates, indicating that \ours can elevate the agent's effective planning capability by enabling the discovery of viable paths in otherwise unsolved instances.

\subsection{Analysis}

\begin{table}[t]
\centering
\caption{\textbf{Planning and sampling error analysis on Sokoban.} We estimate planning and sampling error rates across all states and find that both errors occur in existing frameworks. \ours substantially reduces both errors, showing an improvement in success rate. Best is in \textbf{bold} and second-best is \underline{underlined}.}
\vspace{-5pt}

\label{tab:planning_sampling_errors}
\resizebox{\columnwidth}{!}{
\begin{tabular}{lccc}
\toprule
\textbf{Framework} & \textbf{Planning Error (\%) $\downarrow$} & \textbf{Sampling Error (\%) $\downarrow$} & \textbf{Success Rate (\%) $\uparrow$} \\
\midrule
ReAct  & $50.7 \pm 1.8$ & $8.3 \pm 1.0$ & $5.0 \pm 2.2$ \\
Plan-and-Act     & $\underline{47.7 \pm 1.8}$ & $\underline{4.7 \pm 0.8}$ & $\underline{17.0 \pm 3.8}$ \\
\ours  & $\mathbf{36.7 \pm 1.9}$ & $\mathbf{0.0 \pm 0.0}$ & $\mathbf{46.0 \pm 5.0}$ \\
\bottomrule
\end{tabular}
}
\vspace{-5pt}
\end{table}

\paragraph{Planning \& Sampling Errors.}
We compare planning and sampling errors across agent frameworks, as summarized in \autoref{tab:planning_sampling_errors}.
We use Sokoban for this analysis because it admits an oracle shortest-path planner, which enables us to directly estimate the errors. 
Overall, lower planning and sampling error rates are associated with higher success rates.
This result highlights that \textbf{two sources of errors in \autoref{subsec:error_sources} exist in practice and directly affect task success}.
When comparing ReAct and PA, PA reduces the sampling error by roughly 43.4\% and slightly reduces the planning error by 3.0\%, which is accompanied by a gain in the success rate.
This aligns with our claim in \autoref{prop:pa_vs_react} that providing an in-context plan helps mitigate sampling errors while it does not effectively mitigate planning errors.
Most importantly, \ours significantly reduces planning errors while nearly eliminating sampling errors, leading to the largest improvement in the success rate.
These results empirically support our analysis in \autoref{prop:ours_vs_baselines} that \textbf{plan selection with a solver in a plan graph reduces planning errors and constrained execution suppresses sampling errors}.
The implementation details of our error estimators in Sokoban are provided in \autoref{app:error_estimation}.

\paragraph{Task Difficulty.}
We analyze how the performance of each framework varies with task difficulty as shown in \autoref{fig:main_result}. 
As the difficulty increases, all frameworks exhibit a performance decline.
Specifically, the success rates of ReAct and PA drop by an average of 55.4\% and 52.4\%, respectively.
In contrast, \ours consistently maintains higher success rates across all tasks, with an average decrease of 27.7\%.
This result indicates that \textbf{\ours effectively mitigates planning errors as tasks become hard and a mistaken action becomes more critical}.

\begin{figure}[t]
    \centering
    \resizebox{1.0\linewidth}{!}{
        \begin{subfigure}[b]{0.49\linewidth}
            \centering
            \includegraphics[width=\linewidth]{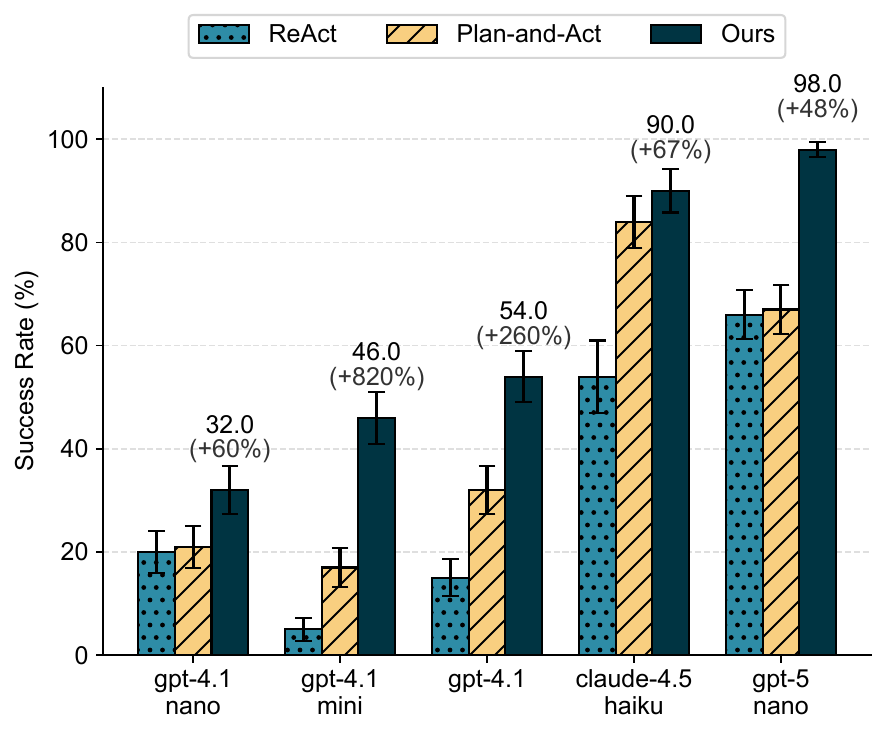}
            \caption{Success Rate across LMs}
            \label{fig:planning_capability_main}
        \end{subfigure}
        \hfill %
        \begin{subfigure}[b]{0.49\linewidth}
            \centering
            \includegraphics[width=\linewidth]{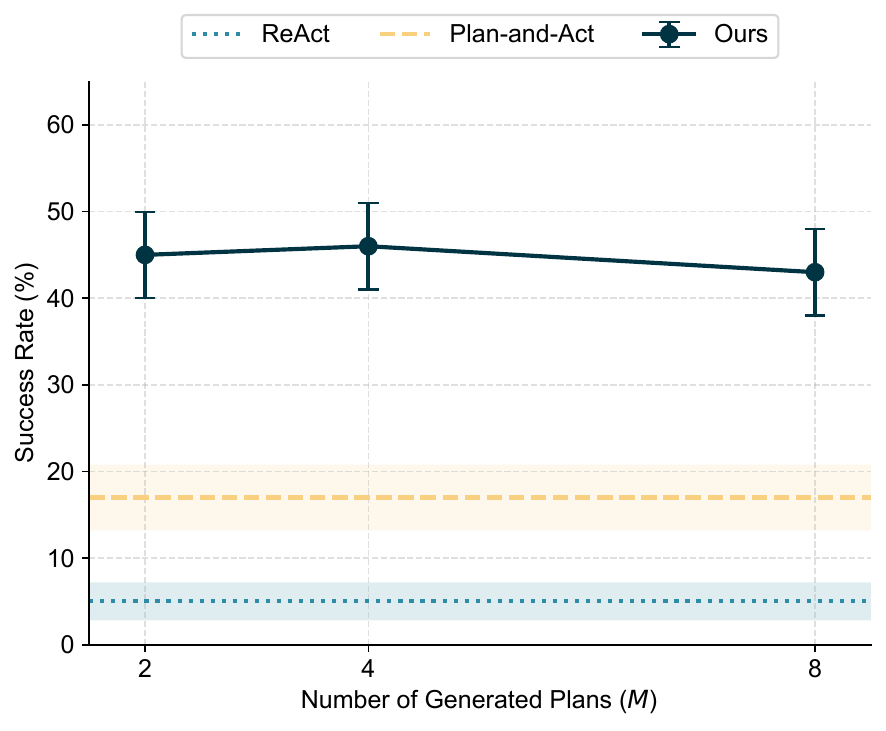}
            \caption{Sensitivity to $M$}
            \label{fig:sensitivity_analysis_main}
        \end{subfigure}
    }
    \caption{
\textbf{Cross-model success rates and sensitivity to the number of generated plans in Sokoban.}
\textbf{(a)} Success rates of \ours and baselines across LMs with different planning capabilities.
\ours consistently improves over other frameworks, with larger gains on weaker models, indicating effective mitigation of planning errors.
\textbf{(b)} Sensitivity of \ours to the number of generated plans $M$ used to construct the plan graph. The best performance is achieved at $M=4$, suggesting that moderate plan aggregation via node merging effectively expands the candidate action space.
}
    \label{fig:main_results}
\end{figure}

\paragraph{Impact of LM's Planning Ability.}
We evaluate \ours across several LMs with varying degrees of reasoning (or planning) capabilities (See \autoref{app:experimental_details} for LM backbones). 
In \autoref{fig:planning_capability_main}, the models are arranged in ascending order of their capability, as evidenced by the monotonic increase in the success rate of ReAct from 22.0\% to 66.0\%. As illustrated in \autoref{fig:planning_capability_main}, \ours consistently outperforms the baselines across all models. 
Notably, our method demonstrates exceptional efficacy on less capable models with limited planning ability; for instance, on \texttt{gpt-4.1-mini}, \ours achieves a relative improvement over ReAct. Even for highly capable models like \texttt{gpt-5-nano}, which already exhibit strong baseline performance, our method yields a significant gain, relatively improving +48\% over ReAct.
This suggests that \textbf{\ours effectively mitigates planning errors due to models with lower reasoning ability}, while eliminating remaining errors in stronger models.

\paragraph{Sensitivity to $M$.} We investigate the impact of the number of generated plans ($M$) used for graph construction on the success rate. As shown in \autoref{fig:sensitivity_analysis_main}, we find that $M=4$ yields the optimal performance. The performance gain observed when increasing $M$ from 2 to 4 is attributed to the aggregation of action candidates; as nodes merge, the number of outgoing edges increases. This reduces the planning errors for each node, supporting \autoref{prop:ours_vs_baselines}. However, we observe a performance decline at $M=8$. This degradation may stem from the reliance on LMs for graph construction. As the number of paths increases, the LLM struggles to maintain global consistency, leading to reduced graph completeness and accumulated construction errors.

\paragraph{Ablation Study.}
We study whether each component of our framework is necessary by selectively removing (i) the \textit{External Solver}, (ii) \textit{Constrained Execution}, and (iii) \textit{Replanning}.
Removing the external solver replaces the formal optimization step that selects a constraint-feasible path on the plan graph (e.g., via an ILP solver) with direct LLM-based selection.
Removing constrained execution means that we only provide the planned path as a prompt-level hint and let the LLM freely generate actions without constrained decoding.
Removing replanning disables mismatch checking and forces the agent to continue executing the planning-time path without adaptation.

As shown in \autoref{tab:sokoban_ablation_easy}, each component of our framework contributes to performance, and the best result is achieved when all three are enabled. Removing the External Solver reduces success from $46.0$ to $42.0$, indicating that formal optimization improves plan quality beyond the LLM’s planning ability. Disabling Constrained Execution causes a larger drop from $46.0$ to $36.0$, showing that prompt-level hints alone do not sufficiently prevent stochastic action deviations. Similarly, removing replanning lowers success to $38.0$, suggesting that mismatch checking and adaptation are important for recovering from execution-time errors and avoiding dead-end trajectories. Overall, the ablation confirms that \textbf{the three components are complementary, addressing planning errors, sampling errors, and plan-execution mismatches}, respectively.
\begin{table}[t]
\centering
\footnotesize
\caption{\textbf{Ablation study.}
We compare \ours against variants that remove components from: (i) \textit{External Solver}, (ii) \textit{Constraint Execution}, and (iii) \textit{Replanning} in Sokoban.
A checkmark (\cmark) indicates the component is enabled; a blank (-) indicates it is removed. Result shows that all components are crucial to improve the success rate. Best is in \textbf{bold} and second-best is \underline{underlined}.}
\vspace{-5pt}

\label{tab:sokoban_ablation_easy}
\resizebox{\columnwidth}{!}{
\begin{tabular}{c c c c}
\toprule
\textbf{External Solver} & \textbf{Constrained Execution} & \textbf{Replanning} & \textbf{Success Rate} \\
\midrule
{\cmark} & {\cmark} & {\cmark} & $\mathbf{46.0 \pm 5.0}$ \\
\midrule
- & {\cmark} & {\cmark} & $\underline{42.0 \pm 4.9}$ \\
{\cmark} & - & {\cmark} & $36.0 \pm 4.8$ \\
{\cmark} & {\cmark} & - & $38.0 \pm 4.8$ \\
\midrule
{\cmark} & - & - & $30.0 \pm 4.9$ \\
- & {\cmark} & - & $19.0 \pm 3.9$ \\
- & - & {\cmark} & $37.0 \pm 4.8$ \\
- & - & - & $11.0 \pm 3.1$ \\
\bottomrule
\end{tabular}
}
\vspace{-15pt}
\end{table}

\section{Conclusion} \label{sec:conclusion}

In this paper, we proposed \textbf{\ours}, a framework designed to mitigate planning and sampling errors in LM agents. By aggregating multiple plans into a graph and employing Integer Linear Programming, \ours identifies feasible paths, thereby reduces the planning error. Furthermore, \ours utilizes constrained execution to substantially reduce sampling error and performs adaptive replanning to handle environment and LM's knowledge mismatches. Our experiments demonstrate that \ours significantly outperforms the ReAct framework and the Plan-and-Act framework, particularly in complex tasks and for models with limited planning capabilities.

\paragraph{Limitations \& Future Work.}
Despite \ours's effectiveness, our framework presents several challenges for future research. First, the accuracy of the plan graph remains dependent on the LM’s ability to correctly structure and merge states. Inaccurate graph construction can lead to a plan space that does not faithfully represent the true environment, suggesting a need for more advanced methods to ensure the structural integrity of the plan graph. 
Second, \ours currently relies on a pre-specified solver, which may limit its generality across tasks with different optimization formulations.
Automatic solver selection based on the task objective is a promising direction for improving generalization.

\section*{Acknowledgements}

This work was supported by the National Science Foundation (NSF) Award DMS-2023239, the NSF CAREER Award CCF-2339978, Amazon Research Award, a grant from FuriosaAI, and Google Cloud Research Credits Program.
We appreciate Jaden Park from University of Wisconsin--Madison,
Chungpa Lee from Yonsei University,
and Minki Kang and Moonseok Choi from KAIST for their valuable discussions.

\section*{Impact Statement}

This work focuses on advancing machine learning by improving the reliability of language model agents that operate under feasibility constraints such as time and cost budgets and tool-usage limits, which may enable more dependable and cost-aware automation and assistance in benign applications. However, increasing agent reliability and effectiveness can lower the barrier to misuse in harmful or unauthorized settings. In addition, our approach may increase inference-time computation due to plan generation, plan-graph construction, which can raise monetary cost and energy use. We encourage responsible deployment with safeguards such as access control, monitoring and logging, and rate limits, and we encourage practitioners to consider cost and efficiency when adopting the method.

\bibliography{icml2026}
\bibliographystyle{icml2026}

\newpage
\appendix
\onecolumn
\section{Conceptual Toy Analysis Details}
\label{app:concept_proof_details}

To illustrate the distinction between planning and sampling errors, we implement the simple ReAct framework in a Sokoban environment, as described in \autoref{sec:problem_formulation}.
At each step, the policy first executes a Breadth-First Search (BFS) algorithm to obtain a shortest plan from the current state and takes its first action as the default intended action.
We then inject planning error with probability $\epsilon_{\mathrm{p}}$ by replacing the intended action with an alternative available action that is \emph{non-viable} under the remaining budget (if any), or otherwise with a random alternative action.
Finally, we inject sampling error with probability $\epsilon_{\mathrm{s}}$ by stochastically flipping the executed action to a different available action, while keeping the intended action unchanged.
Thus, $\epsilon_{\mathrm{p}}$ controls errors in action selection at planning time, whereas $\epsilon_{\mathrm{s}}$ controls stochastic deviations at execution time.
In our Sokoban setting, we set the action budget to the optimal solution length plus a slack of $2$ steps.
In \autoref{fig:motivation}, we set $\epsilon_{\mathrm{p}}=0.25$ (planning error) and $\epsilon_{\mathrm{s}}=0.20$ (sampling error).

\section{Related Work} \label{app:related_work}

\paragraph{Agentic Tasks.} %
Agentic tasks refer to problems in which an agent interacts with the external environment to accomplish the goal~\cite{ReAct}. Recent advancements have demonstrated the possibility of LMs' agentic ability across diverse domains, including robotic  manipulation~\citep{RobotLLMSurvey},
graphical user interface (GUI) navigation~\citep{Webvoyager},
and software engineering~\citep{SWEBench, Sweagent}.
Specifically,
in life simulations,
agents typically role-play as Non-Player Characters (NPCs), exhibiting human-like behaviors~\citep{Smallville, InfectedSmallVille}.

However, the impressive capabilities of these agents often rely on unrestricted resource consumption, which poses a significant barrier to practical deployment.
For instance, \citet{Smallville} shows that a two-day simulation of 25 agents incurred thousands of dollars in API costs.
Similarly, subsequent studies noted that the latency induced by such complex reasoning loops renders real-time interaction infeasible~\citep{Humanoidagents, Agentsims, Agentsociety}.
These examples highlight that in real-world scenarios, agents cannot be deployed with unbounded resources.

Instead, agents must operate under constraints, such as monetary budgets, latency limits, or safety protocols~\citep{BudgetTool}.
Operating under these constraints fundamentally changes the problem: the agent can no longer rely on exhaustive trial-and-error, and small planning or sampling errors can compound into constraint violations and irrecoverable states~\citep{LLMPlanningCapability, ETAL}.
Our work focuses on improving the success rate of LM agents in such constrained settings by mitigating such error compounding.

\paragraph{Language Model Agents.}
LM agents solve complex tasks by interacting with external environments, such as computers and databases, via tools~\citep{ReAct, ToolFormer}. The ReAct framework~\citep{ReAct} established a baseline by interleaving reasoning and acting, allowing agents to respond to immediate observations. However, relying solely on step-by-step generation often makes agents susceptible to planning errors~\citep{LLMPlanningCapability} and sampling errors~\citep{ETAL}.

To reduce both errors, some works have been proposed. To reduce planning errors, some works incorporate trained world knowledge models for planning~\citep{WKM, DGAP}, search over multiple candidate thoughts~\citep{ToT}, perform deliberation over candidate action sequences using imagined rollouts or heuristic search~\citep{RAP, RAFA, ToolChain, ThoughtOfSearch}, or use reflection mechanisms to identify failures and revise subsequent behavior~\citep{Reflexion, SAND, ReflAct}.
To mitigate sampling errors, Plan-and-Solve strategies generate a high-level plan before execution to improve global coherence and reduce logical inconsistencies during execution~\citep{PlanAndSolve, PlanAndAct, ReWoo, MPO}.
Some works address both errors by combining planning before execution with online plan refinement during execution~\citep{AdaPlanner, ADaPT, WebPilot}.
Despite these improvements, these methods typically do not enforce hard feasibility constraints during plan selection, and execution can still suffer from sampling errors even when the plan is feasible, since action generation remains stochastic and is not constrained to follow the plan.

In contrast, our framework enforces hard feasibility at plan selection by using a formal solver to select a constraint-feasible path over a plan graph constructed from multiple candidate plans, thereby reducing the planning errors.
It further reduces sampling errors by constraining decoding to the selected action, with replanning when observations mismatch the plan.

\paragraph{Neural-Symbolic Approaches.}
Neural-symbolic approaches integrate external solvers into learning-based models to incorporate symbolic structure and formal constraints~\citep{NeuralSymbolic}. 
For LLMs, LMs usually act as a translator that converts natural-language task descriptions into formal planning inputs for external solvers, such as the Planning Domain Definition Language (PDDL)~\citep{LLM+P, DynamicPlanningLLM, PDDLPlanning, CaStL, PlanningInTheDark, ThoughtOfSearch}, Satisfiability Modulo Theories (SMT)~\citep{ETAL, PlanningRigor}, or Linear Programming (LP)~\citep{LiP-LLM}.
Most prior neural-symbolic LM planning frameworks use LMs primarily as translators that map natural language into a single formal specification or solver-consistent plan, and treat execution as a downstream procedure.

In contrast, our framework keeps a diverse set of candidate plans, folds them into a plan graph, and solves a constraint-feasible path selection problem over this graph.
We further integrate planning and execution by constraining decoding to the selected action and replanning upon plan-observation mismatches, which is critical under strict feasibility constraints where deviations are irrecoverable.

\section{ILP Formulation for Budget Constrained Setting}\label{app:constrained_ilp}

In the constrained G-MDP setting, the agent aims to maximize the expected reward while strictly adhering to a specified budget (e.g., token limit or search depth). We extend the path selection formulation presented in \autoref{sec:path_selection} by incorporating the predicted edge costs.
Recall that for each edge $e \in E$, our model estimates a cost $\hat{c}_\theta(e)$, representing the expected resource consumption of traversing that edge. Let $\mathbf{ b}_t$ denote the remaining budget vector at the current step. We formulate the constrained path selection problem by adding a linear budget constraint to the original ILP:

\begin{align}
\max_{x}\quad& \sum_{\ell=0}^{L_{\max}-1}\ \sum_{e\in E} \hat r_\theta\big(\mathrm{tgt}(e)\big),x_{e,\ell} \label{eq:app_obj} \\
\text{s.t.} & \sum_{\ell=0}^{L_{\max}-1} \sum_{e\in E} \hat{\mathbf{c}}_\theta(e),x_{e,\ell} \preceq \mathbf{b}_t, \label{eq:app_budget} \\
&\text{Constraints \eqautoref{eq:single_action} to \eqautoref{eq:binary_var}}. \notag
\end{align}

Here, \eqautoref{eq:app_budget} enforces the \textit{budget constraint}, ensuring that the cumulative predicted cost of the selected walk does not exceed $\mathbf b_t$. The remaining structural constraints (\eqautoref{eq:single_action} to \eqautoref{eq:binary_var}) are identical to those in the unconstrained formulation, guaranteeing that the solution forms a valid directed walk from the start node $v_0$ to a goal node in $V_g$.If the solver finds no solution satisfying \eqautoref{eq:app_budget} (i.e., all valid paths to the goal exceed the budget), it returns an infeasibility status, which can be handled by a fallback policy or by re-planning with relaxed constraints.

\section{Proofs of Theoretical Analysis}\label{app:proofs}

In this section, we provide detailed proofs for the propositions presented in \autoref{sec:theoretical_analysis}.
We define the event of selecting a viable action at step $t$ as $\mathcal{V}_t := \{a_t \in \mathcal{A}_{\mathrm{viable}}(s_t)\}$.
According to \autoref{def:success_event}, the success probability for a task requiring $T_g$ steps is given by $\Pr(\mathcal{S}) = \Pr(\cap_{t=0}^{T_g-1} \mathcal{V}_t)$.
Assuming the Markov property and independence of errors at each step, we focus on deriving the single-step success probability $P(\mathcal{V}_t)$.

\subsection{Proof of \autoref{prop:react_success_rate}}
\begin{proof}
Let $\hat{a}_t$ be the action planned by the agent (intent), and $a_t$ be the action actually executed.
Also, let $E_{\mathrm{plan}}$ be the event that the planned action is viable, i.e., $\hat{a}_t \in \mathcal{A}_{\mathrm{viable}}(s_t)$, and let $E_{\mathrm{exec}}$ be the event that the executed action matches the plan, i.e., $a_t = \hat{a}_t$.
We define two independent events based on the error sources as follows: $\Pr(E_{\mathrm{plan}}) = 1 - \epsilon_{\mathrm{p}}$ and $\Pr(E_{\mathrm{plan}}^c) = \epsilon_{\mathrm{p}}$, and $\Pr(E_{\mathrm{exec}}^c) = \epsilon_{\mathrm{s}}$.

The viable action event $\mathcal{V}_t$ can occur in two disjoint cases.
First, when the plan is viable, the executed action $a_t$ remains viable if execution succeeds ($E_{\mathrm{exec}}$) or if execution fails but does not break viability. The latter occurs with probability $1-\delta_{\mathrm{b}}$.
\begin{equation}
\Pr(\mathcal{V}_t \mid E_{\mathrm{plan}}) = (1-\epsilon_{\mathrm{s}}) \cdot 1 + \epsilon_{\mathrm{s}} \cdot (1-\delta_{\mathrm{b}}) = 1 - \epsilon_{\mathrm{s}}\delta_{\mathrm{b}}.
\end{equation}
Second, when the plan is non-viable, the executed action $a_t$ becomes viable only if execution fails ($E_{\mathrm{exec}}^c$) and accidentally recovers viability (with probability $\delta_{\mathrm{r}}$).
\begin{equation}
    \Pr(\mathcal{V}_t \mid E_{\mathrm{plan}}^c) = (1-\epsilon_{\mathrm{s}}) \cdot 0 + \epsilon_{\mathrm{s}} \cdot \delta_{\mathrm{r}} = \epsilon_{\mathrm{s}}\delta_{\mathrm{r}}.
\end{equation}

By the Law of Total Probability, the probability of selecting a viable action at step $t$ is:
\begin{align}
\Pr(\mathcal{V}_t) &= \Pr(\mathcal{V}_t \mid E_{\mathrm{plan}})\Pr(E_{\mathrm{plan}}) + \Pr(\mathcal{V}_t \mid E_{\mathrm{plan}}^c)\Pr(E_{\mathrm{plan}}^c) \\
&= (1-\epsilon_{\mathrm{s}}\delta_{\mathrm{b}})(1-\epsilon_{\mathrm{p}}) + (\epsilon_{\mathrm{s}}\delta_{\mathrm{r}})\epsilon_{\mathrm{p}} \\
&= (1-\epsilon_{\mathrm{p}})(1-\epsilon_{\mathrm{s}}\delta_{\mathrm{b}}) + \epsilon_{\mathrm{p}}\epsilon_{\mathrm{s}}\delta_{\mathrm{r}}.
\end{align}
Since the task requires at least $T \le T_g$ steps, the overall success probability is the product of single-step probabilities up to $T$:
\begin{align}
\Pr(\mathcal{S}) &\le \prod_{t=0}^{T-1} \Pr(\mathcal{V}_t) \\
&= \Big((1-\epsilon_{\mathrm{p}})(1-\epsilon_{\mathrm{s}}\delta_{\mathrm{b}}) + \epsilon_{\mathrm{p}}\epsilon_{\mathrm{s}}\delta_{\mathrm{r}}\Big)^{T} \\
&=: U_{\mathrm{ReAct}}.
\end{align}
Equality holds if the task ends exactly at step $T$, i.e. $T=T_g$.

To show monotonicity, let $f(\epsilon_{\mathrm{s}}) = U_{\mathrm{ReAct}}$.
We take the derivative with respect to $\epsilon_{\mathrm{s}}$
\begin{equation}
\frac{\partial}{\partial \epsilon_{\mathrm{s}}} f(\epsilon_{\mathrm{s}}) = \left( -(1-\epsilon_{\mathrm{p}})\delta_{\mathrm{b}} + \epsilon_{\mathrm{p}}\delta_{\mathrm{r}}\right)^T.
\end{equation}

For the success probability to increase as $\epsilon_{\mathrm{s}}$ decreases (i.e., derivative is negative), we require

\begin{equation}
-(1-\epsilon_{\mathrm{p}})\delta_{\mathrm{b}} + \epsilon_{\mathrm{p}}\delta_{\mathrm{r}} \le 0 \iff (1-\epsilon_{\mathrm{p}})\delta_{\mathrm{b}} \ge \epsilon_{\mathrm{p}}\delta_{\mathrm{r}}.
\end{equation}

Similarly, we can show that the upper bound of success probability increases as $\epsilon_{\mathrm{p}}$ decreases. This confirms the condition stated in \autoref{prop:react_success_rate}.
\end{proof}

\subsection{Proof of \autoref{prop:pa_vs_react}}

\begin{proof}
According to the Plan-and-Act mechanism defined in \eqautoref{eq:plan_and_act}, the agent behavior is divided into two cases based on plan alignment and following probability.

First, following the plan occurs when the plan is aligned ($z(s_t)=\hat s_t^{\hat \tau}$) and the agent chooses to follow it. We denote this event as $F$. The probability of this event is $\Pr(F) = \alpha p_{\mathrm{follow}}$.
In this case, $a_t = \bar{a}_t$.
Since $\bar{a}_t$ is deterministically selected from the pre-generated plan, the execution sampling error is eliminated ($\epsilon_{\mathrm{s}} \to 0$). However, the pre-generated plan itself is subject to the same planning error $\epsilon_{\mathrm{p}}$ as the ReAct framework.
Thus, the success probability for this case is derived by setting $\epsilon_{\mathrm{s}}=0$ in the ReAct single-step probability as
\begin{equation}
\Pr(\mathcal{V}_t \mid F) = (1-\epsilon_{\mathrm{p}})(1-0) + \epsilon_{\mathrm{p}}(0) = 1-\epsilon_{\mathrm{p}}.
\end{equation}

Second, fallback to ReAct occurs when the plan is not aligned or the agent chooses not to follow. We denote this event as $R$. The probability is $\Pr(R) = 1 - \alpha p_{\mathrm{follow}}$.
We denote this event as $R$.
In this case, $a_t = a_t^{\mathrm{ReAct}}$, and the success probability is identical to that of the ReAct framework derived in \autoref{prop:react_success_rate} as
\begin{equation}
    \Pr(\mathcal{V}_t \mid R) = (1-\epsilon_{\mathrm{p}})(1-\epsilon_{\mathrm{s}}\delta_{\mathrm{b}}) + \epsilon_{\mathrm{p}}\epsilon_{\mathrm{s}}\delta_{\mathrm{r}}.
\end{equation}

By the Law of Total Probability, the single-step success probability for PA is
\begin{align}
\Pr(\mathcal{V}_t) &= \Pr(F)\Pr(\mathcal{V}_t \mid F) + \Pr(R)\Pr(\mathcal{V}_t \mid R) \\
&= (\alpha p_{\mathrm{follow}})(1-\epsilon_{\mathrm{p}}) + (1-\alpha p_{\mathrm{follow}})\Big[ (1-\epsilon_{\mathrm{p}})(1-\epsilon_{\mathrm{s}}\delta_{\mathrm{b}}) + \epsilon_{\mathrm{p}}\epsilon_{\mathrm{s}}\delta_{\mathrm{r}} \Big].
\end{align}
We can rearrange the terms to isolate $\epsilon_{\mathrm{s}}$:
\begin{align}
\Pr(\mathcal{V}_t) &= (\alpha p_{\mathrm{follow}})(1-\epsilon_{\mathrm{p}}) + (1-\alpha p_{\mathrm{follow}})(1-\epsilon_{\mathrm{p}}) + (1-\alpha p_{\mathrm{follow}})\left[ -(1-\epsilon_{\mathrm{p}})\epsilon_{\mathrm{s}}\delta_{\mathrm{b}} + \epsilon_{\mathrm{p}}\epsilon_{\mathrm{s}}\delta_{\mathrm{r}} \right] \\
&= (1-\epsilon_{\mathrm{p}}) - (1-\epsilon_{\mathrm{p}})(1-\alpha p_{\mathrm{follow}})\epsilon_{\mathrm{s}}\delta_{\mathrm{b}} + \epsilon_{\mathrm{p}}(1-\alpha p_{\mathrm{follow}})\epsilon_{\mathrm{s}}\delta_{\mathrm{r}}. \label{eq:prob_w_o_sampling_error_pa}
\end{align}

Now, we define the effective sampling error for PA as $\epsilon_{\mathrm{s,PA}} \coloneqq (1-\alpha p_{\mathrm{follow}})\epsilon_{\mathrm{s}}$. Substituting this into \eqautoref{eq:prob_w_o_sampling_error_pa}, we have
\begin{equation}
\Pr(\mathcal{V}_t) = (1-\epsilon_{\mathrm{p}})(1 - \epsilon_{\mathrm{s,PA}}\delta_{\mathrm{b}}) + \epsilon_{\mathrm{p}}\epsilon_{\mathrm{s,PA}}\delta_{\mathrm{r}}.
\end{equation}

Finally, since the task requires $T$ steps, we take the product over $t=0$ to $T-1$, we have
\begin{equation}
U_{\mathrm{PA}} = \Big((1-\epsilon_{\mathrm{p}})(1-\epsilon_{\mathrm{s,PA}}\delta_{\mathrm{b}}) + \epsilon_{\mathrm{p}}\epsilon_{\mathrm{s,PA}}\delta_{\mathrm{r}}\Big)^{T}.
\end{equation}
The comparison $U_{\mathrm{PA}} \ge U_{\mathrm{ReAct}}$ follows directly from $\epsilon_{\mathrm{s,PA}} \le \epsilon_{\mathrm{s}}$ and equality holds when $\alpha p_{\mathrm{follow}}=0$.
\end{proof}

\subsection{Proof of \autoref{prop:ours_vs_baselines}}

\begin{proof}

Let $\hat{\mathcal{A}}(v_t)$ be the set of $d(v_t)$ candidate actions generated at step $t$.
A planning failure occurs for the entire set only if \textit{all} candidates in $\hat{\mathcal{A}}(v_t)$ are non-viable. Assuming the generation of each candidate is independent given the state, the probability that all candidates fail is $(\epsilon_{\mathrm{p}})^{d(v_t)}$.
Consequently, the probability that there exists at least one viable action in the candidate set is:
\begin{equation}
    \Pr(\exists a \in \hat{\mathcal{A}}(v_t) : a \in \mathcal{A}_{\mathrm{viable}}) = 1 - (\epsilon_{\mathrm{p}})^{d(v_t)}.
\end{equation}
Based on the assumptions in \autoref{prop:ours_vs_baselines}, the external solver successfully identifies a viable action if one exists in the set and constrained execution eliminates sampling error (i.e., $\epsilon_{\mathrm{s}} \approx 0$), ensuring the selected action is executed exactly.
Thus, the single-step success probability for \ours is exactly $1 - (\epsilon_{\mathrm{p}})^{d(v_t)}$. Taking the product over $T$ steps yields
\begin{equation}
    U_{\mathrm{Ours}} = \prod_{t=0}^{T-1} \Big( 1 - (\epsilon_{\mathrm{p}})^{d(v_t)} \Big).
\end{equation}

From \autoref{prop:pa_vs_react}, we already established $U_{\mathrm{PA}} \ge U_{\mathrm{ReAct}}$. We now focus on proving $U_{\mathrm{Ours}} \ge U_{\mathrm{PA}}$.

Since we assume $(1-\epsilon_{\mathrm{p}})\delta_{\mathrm{b}} \ge \epsilon_{\mathrm{p}}\delta_{\mathrm{r}}$, the success probability of PA at step $t$ is maximized when the effective sampling error is zero ($\epsilon_{\mathrm{s,PA}}=0$).
Substituting $\epsilon_{\mathrm{s,PA}}=0$ into the term for PA gives the upper bound:
\begin{equation}
    (1-\epsilon_{\mathrm{p}})(1-\epsilon_{\mathrm{s,PA}}\delta_{\mathrm{b}}) + \epsilon_{\mathrm{p}}\epsilon_{\mathrm{s,PA}}\delta_{\mathrm{r}} \le 1 - \epsilon_{\mathrm{p}}. \label{eq:inequality_stepwise_1}
\end{equation}
Next, for \ours, since $d(v_t) \ge 1$ and $0 \le \epsilon_{\mathrm{p}} < 1$, it follows that $(\epsilon_{\mathrm{p}})^{d(v_t)} \le \epsilon_{\mathrm{p}}$.
Therefore, we have
\begin{equation}
    1 - (\epsilon_{\mathrm{p}})^{d(v_t)} \ge 1 - \epsilon_{\mathrm{p}}. \label{eq:inequality_stepwise_2}
\end{equation}
Combining \eqautoref{eq:inequality_stepwise_1} and \eqautoref{eq:inequality_stepwise_2}, we have
\begin{equation}
    1 - (\epsilon_{\mathrm{p}})^{d(v_t)} \ge 1 - \epsilon_{\mathrm{p}} \ge (1-\epsilon_{\mathrm{p}})(1-\epsilon_{\mathrm{s,PA}}\delta_{\mathrm{b}}) + \epsilon_{\mathrm{p}}\epsilon_{\mathrm{s,PA}}\delta_{\mathrm{r}}.
    \label{eq:inequality_stepwise}
\end{equation}
Since \eqautoref{eq:inequality_stepwise} satisfies for all steps $t=0\dots T-1$, we conclude $U_{\mathrm{Ours}} \ge U_{\mathrm{PA}}$. Also, we can see that $1 - (\epsilon_{\mathrm{p}})^{d(v_t)}=(1-\epsilon_{\mathrm{p}})(1-\epsilon_{\mathrm{s,PA}}\delta_{\mathrm{b}}) + \epsilon_{\mathrm{p}}\epsilon_{\mathrm{s,PA}}\delta_{\mathrm{r}}$ holds when $d(v_t)=1$ and $\epsilon_{\mathrm{s,PA}}\delta_{\mathrm{b}} = \epsilon_{\mathrm{s,PA}}\delta_{\mathrm{r}} = 0$.
\end{proof}

\newpage

\section{Experiment Details}
\subsection{Experimental Setup}\label{app:experimental_details}

\paragraph{Tasks and Datasets} Tasks and Datasest details are explained below. The detailed statistics are summarized in Table.
\begin{itemize}
    \item \textbf{Sokoban~\citep{Sokoban}.} is a task pushing all boxes onto goal tiles using four primitive actions (\texttt{U}, \texttt{D}, \texttt{L}, \texttt{R}). To make the task irrecoverable, we define success as solving the instance within two steps of the optimal solution length. We construct two difficulty by the controlling optimal step count $T^\star$: \emph{easy} instances have $T^\star = 6$, while \emph{hard} instances have $T^\star = 10$. For evaluation, we construct 10 maps for each task difficulty and run 10 times for each map.

    \item \textbf{ALFWorld~\citep{ALFWorld}.} This task is a synthetic text-based household embodied tasks where the agent executes environment-specific actions to complete a given goal. To align with our assumption that the agent knows the abstract world model in agentic tasks, we modify the environment to provide the object availability information for each location and the basic information of the dependency of the action. We define difficulty by the action budget relative to the optimal length $T^\star$: \emph{easy} instances use a looser budget, while \emph{hard} instances use a tighter budget. For evaluation, 

    \item \textbf{GSM-Hard~\citep{PAL}} is a mathemathical reasoning problem dataset that converts some numerical values larger so that the problem cannot be solve easily without arithmetic tools.
We cast as an agentic task by equipping the agent with arithmetic tools (\texttt{+}, \texttt{-}, \texttt{$\times$}, \texttt{/}).
For each operator, we provide two tool variants: a \emph{fast} tool with lower success probability and a \emph{slow} tool with higher success probability.
We define \emph{hard} tasks as those with a tight time budget that incentivizes using fast tools, and \emph{easy} tasks as those with a loose time budget.

\item \textbf{MuSiQue~\citep{MuSiQue}} is a factual multi-hop reasoning dataset, where the agent can query a retriever to obtain supporting information.
To create an agentic setting with explicit cost--quality tradeoffs, we construct five synthetic retrievers ranging from fast, cheap, but inaccurate to slow, expensive, but accurate.
Analogous to GSM-8K, we define \emph{hard} tasks as those requiring fast and cheap solving, and \emph{easy} tasks as those allowing a larger time budget and higher retrieval cost.
\end{itemize}

\paragraph{Language Model Backbones.} We use five LM backbones: \texttt{gpt-4.1-nano}, \texttt{gpt-4.1-mini}, \texttt{gpt-4.1}~\citep{openai2025gpt41}, \texttt{gpt-5-nano}~\citep{singh2025openai}, and \texttt{claude-4.5-haiku}~\citep{anthropic2025system_card}.

\paragraph{Inference} We set the sampling temperature to $0.3$ for inference in all experiments except \texttt{gpt-5-nano} models. For \texttt{gpt-5-nano}, as we do not change the temperature, we set the default value (it is unknown). Other parameters, such as top-$p$ and repetition penaltiy, are set as the default value (both are $1$).

\subsection{Baselines}
\paragraph{ReAct~\citep{ReAct}} ReAct is the framework that make LM agents solve the tasks by interleaving thought and actions to interact with the external environments. There are various types of implementation for the thoughts and acts~\citep{ReAct, Reflexion, ReflAct}, we adopt the prompting technique from \citet{MPO}. Detailed prompts are summarized in Prompt~\ref{box:prompt_react_sokoban}.

\paragraph{Plan-and-Act~\citep{PlanAndSolve, PlanAndAct}}
Plan-and-Act (PA) first generates the a plan and use it as in-context prompt. Then, LM agent refers the plan to interact with the environment. For inteaction part, the prompt is the same as ReAct. Detailed prompts for the plan generation phase are summarized in Prompt~\ref{box:prompt_pa_sokoban}.

\newpage
\begin{promptbox}[label={box:prompt_react_sokoban}]{ReAct (Sokoban)}
Interact with Sokoban environment to solve a task (placing every box onto goals.)

\vspace{1em}

\#\# Sokoban rules (task + mechanics)

- Task objective: place every box onto goals. The puzzle is solved when all boxes are on goals.

- Action mechanics: each action moves the player exactly one cell in the chosen direction (U/D/L/R).

    - U: move up, e.g., (x, y) -> (x, y + 1)
    
    - D: move down, e.g., (x, y) -> (x, y - 1)

    - L: move left, e.g., (x, y) -> (x - 1, y)

    - R: move right, e.g., (x, y) -> (x + 1, y)

    - The action name must match the coordinate change (U increases y, D decreases y, R increases x, L decreases x).

- Walls: the player cannot move into a wall cell. If a wall occupies the destination cell, the action is invalid and the player does not move.

- Boxes: if the destination cell has a box, the player attempts to push it one cell further in the same direction.

    - The push succeeds only if the cell behind the box is empty floor or a goal.

    - If the cell behind the box is a wall or another box, the push is invalid and nothing moves.

    - A push is only possible when the player is on the cell immediately adjacent to the box from the opposite side of the push direction.

- The player cannot pull boxes and cannot push two boxes at once.

- Some pushes can create deadlocks (for example, pushing a box into a corner where it cannot reach any goal).

Examples (coordinate outcomes, using U/D/L/R only):

- Empty move, R: player at (x, y), no wall/box at (x + 1, y). Action R -> player at (x + 1, y); box on goal location unchanged (goals are unaffected by moves).

- Empty move, L: player at (x, y), no wall/box at (x - 1, y). Action L -> player at (x - 1, y); box on goal location unchanged (goals are unaffected by moves).

- Empty move, U: player at (x, y), no wall/box at (x, y + 1). Action U -> player at (x, y + 1); box on goal location unchanged (goals are unaffected by moves).

- Empty move, D: player at (x, y), no wall/box at (x, y - 1). Action D -> player at (x, y - 1); box on goal location unchanged (goals are unaffected by moves).

- Wall block, U: player at (x, y), wall at (x, y + 1). Action U -> invalid, player stays at (x, y); box on goal location unchanged.

- Wall block, D: player at (x, y), wall at (x, y - 1). Action D -> invalid, player stays at (x, y); box on goal location unchanged.

- Wall block, L: player at (x, y), wall at (x - 1, y). Action L -> invalid, player stays at (x, y); box on goal location unchanged.

- Wall block, R: player at (x, y), wall at (x + 1, y). Action R -> invalid, player stays at (x, y); box on goal location unchanged.

- Push succeeds, R: player at (x, y), box at (x + 1, y), no wall/box at (x + 2, y). Action R -> player at (x + 1, y), box moves to (x + 2, y); if (x + 2, y) is a goal, box on goal location becomes (x + 2, y), otherwise unchanged.

- Push succeeds, L: player at (x, y), box at (x - 1, y), no wall/box at (x - 2, y). Action L -> player at (x - 1, y), box moves to (x - 2, y); if (x - 2, y) is a goal, box on goal location becomes (x - 2, y), otherwise unchanged.

- Push succeeds, U: player at (x, y), box at (x, y + 1), no wall/box at (x, y + 2). Action U -> player at (x, y + 1), box moves to (x, y + 2); if (x, y + 2) is a goal, box on goal location becomes (x, y + 2), otherwise unchanged.

- Push succeeds, D: player at (x, y), box at (x, y - 1), no wall/box at (x, y - 2). Action D -> player at (x, y - 1), box moves to (x, y - 2); if (x, y - 2) is a goal, box on goal location becomes (x, y - 2), otherwise unchanged.

- Push blocked by wall, R: player at (x, y), box at (x + 1, y), wall at (x + 2, y). Action R -> invalid, player stays at (x, y), box stays at (x + 1, y); box on goal location unchanged.

- Push blocked by wall, L: player at (x, y), box at (x - 1, y), wall at (x - 2, y). Action L -> invalid, player stays at (x, y), box stays at (x - 1, y); box on goal location unchanged.

- Push blocked by wall, U: player at (x, y), box at (x, y + 1), wall at (x, y + 2). Action U -> invalid, player stays at (x, y), box stays at (x, y + 1); box on goal location unchanged.

- Push blocked by wall, D: player at (x, y), box at (x, y - 1), wall at (x, y - 2). Action D -> invalid, player stays at (x, y), box stays at (x, y - 1); box on goal location unchanged.

- Push blocked by box, R: player at (x, y), box at (x + 1, y), another box at (x + 2, y). Action R -> invalid, player stays at (x, y), boxes stay at (x + 1, y) and (x + 2, y); box on goal location unchanged.

- Push blocked by box, L: player at (x, y), box at (x - 1, y), another box at (x - 2, y). Action L -> invalid, player stays at (x, y), boxes stay at (x - 1, y) and (x - 2, y); box on goal location unchanged.

- Push blocked by box, U: player at (x, y), box at (x, y + 1), another box at (x, y + 2). Action U -> invalid, player stays at (x, y), boxes stay at (x, y + 1) and (x, y + 2); box on goal location unchanged.

- Push blocked by box, D: player at (x, y), box at (x, y - 1), another box at (x, y - 2). Action D -> invalid, player stays at (x, y), boxes stay at (x, y - 1) and (x, y - 2); box on goal location unchanged.

\vspace{1em}

\#\# Instruction

You are the player in a Sokoban environment, and your goal is to place every box on a goal within a limited number of actions (within step remaining). That means you need to make the same number of boxes on goals by placing all boxes onto goals. If a box is on a goal, it is considered satisfied.

At each turn, you will receive the current observation.

Observation format (all coordinates are (x, y)):

- wall location: (x1, y1), (x2, y2), ...

- player location: (x, y)

- box location: (x3, y3), ...

- goal location: (x4, y4), ...

- box on goal location: (x5, y5), ...

- Step remaining: <steps\_remaining>

The "box location" list includes only boxes not on goals. The "goal location" list includes all goals.
You may choose between two outputs: "Thought" or "Action".
\vspace{1em}

When you choose "Thought", you must:

Plan the full solution (overall path) so that all boxes reach goals within the remaining steps, using the Sokoban rules (task + mechanics).

From that full plan, explicitly predict only the next 1 step: how the immediate action will change the player location, box location, and box on goal location.

Based on that planning and the current observation, decide the immediate next action to take.

IMPORTANT:

    - Your Thought MUST match the given observation exactly. No hallucination about positions or adjacency is allowed.
    
    - If a push is feasible immediately, you MUST choose the move that pushes.
    
    - If you planned a push in the previous Thought and the current observation still allows that push, you MUST continue and execute it.
    
    - Do NOT rewrite the plan from scratch unless the environment changed and the old plan became infeasible.

Your output must follow exactly:

Thought: <your reasoning>

Action: <U|D|L|R>"
\end{promptbox}

\begin{promptbox}[label={box:prompt_pa_sokoban}]{Plan-and-Act (Sokoban)}
Interact with Sokoban environment to solve a task (placing every box onto goals.)

\{\{sokoban\_rule\}\}
\vspace{1em}

\#\# Instruction
You are the player in a Sokoban environment, and your goal is to place every box on a goal within a limited number of actions (within step remaining). That means you need to make the same number of boxes on goals by placing all boxes onto goals. If a box is on a goal, it is considered satisfied.

At each turn, you will receive the current observation.
Observation format (all coordinates are (x, y)):
- wall location: (x1, y1), (x2, y2), ...
- player location: (x, y)
- box location: (x3, y3), ...
- goal location: (x4, y4), ...
- box on goal location: (x5, y5), ...
- Step remaining: <steps\_remaining>
The "box location" list includes only boxes not on goals. The "goal location" list includes all goals.

You are in planning mode. Using the Sokoban rules, plan the full solution (overall path) so that all boxes reach goals within the remaining steps, using the Sokoban rules (task + mechanics).
Based on this plan, you generate the full sequence of actions.
\end{promptbox}

\newpage

\section{Implementation Details of \ours}
In this section, we introduce the pseudo-code for \ours and prompt details for our implementation. Detailed implementation of our framework is available on \href{https://github.com/UW-Madison-Lee-Lab/TAPE}{Github}.

\subsection{Pseudo-code}
The overall procedure of \ours is detailed in \autoref{alg:ours}.

\begin{algorithm}[t]
\caption{Tool-Guided Adaptive Planning with Constrained Execution (\ours)}
\label{alg:ours}
\begin{algorithmic}[1]
\REQUIRE Environment $\mathcal{E}$, budgets $B$, max horizon $L_{\mathrm{max}}$, \#candidates $M$
\REQUIRE LLM-based abstract state projector $z_\theta(\cdot)$, LLM-based plan graph constructor $\textsc{BuildPlanGraph}(\cdot)$
\REQUIRE LLM predictors $\textsc{AnnotateGraph}(\cdot)$ for node reward / edge cost
\REQUIRE External solver $\textsc{Solver}(\cdot)$, constrained decoding method $\textsc{ConstrainedDecoding}(\cdot)$
\STATE Initialize history $s_0 \leftarrow (\,)$, observe $o_0 \leftarrow \mathcal{E}.\textsc{Reset}()$, $t \leftarrow 0$

\WHILE{$t < L_{\mathrm{max}}$ \AND not terminal}
    \STATE $\hat s_t \leftarrow z_\theta(s_t)$ \COMMENT{$\triangleright$ State projection}
    
    \STATE $\{\hat\tau^{(m)}\}_{m=1}^{M} \leftarrow \textsc{SamplePlans}(\hat s_t, M)$ \COMMENT{$\triangleright$ Candidate rollout sampling}
    
    \STATE $G \leftarrow \textsc{BuildPlanGraph}(\{\hat\tau^{(m)}\}, z_\theta)$ \COMMENT{$\triangleright$ Graph construction via abstract state merging}
    
    \STATE $G \leftarrow \textsc{AnnotateGraph}(G)$ \COMMENT{$\triangleright$ LLM-based reward/cost annotation}
    
    \STATE $\pi^\star \leftarrow \textsc{Solver}(G, \hat s_t, B, L-t)$ \COMMENT{$\triangleright$ Feasible plan path selection via ILP}
    
    \WHILE{$t < L$ \AND not terminal}
        \STATE $a_t^\star \leftarrow \pi^\star[t]$
        \STATE $a_t \leftarrow \textsc{ConstrainedDecoding}(s_t, a_t^\star)$ \COMMENT{$\triangleright$ Enforce sampling constraint}
        \STATE $(o_{t+1}, \text{cost}_t, \text{done}) \leftarrow \mathcal{E}.\textsc{Step}(a_t)$
        \STATE $s_{t+1} \leftarrow s_t \cup (a_t, o_{t+1}), \;\; B \leftarrow B - \text{cost}_t$
        
        \STATE $\hat s_{t+1} \leftarrow z_\theta(s_{t+1}), \;\; \hat v_{t+1}^\star \leftarrow \textsc{NextPlannedNode}(G,\pi^\star,t)$
        
        \IF{$\hat s_{t+1} \neq \hat s_{t+1}^\star$ \OR $B < 0$}
            \STATE \textbf{break} \COMMENT{$\triangleright$ Replan upon deviation or budget violation}
        \ENDIF
        \STATE $t \leftarrow t+1$
    \ENDWHILE
\ENDWHILE
\end{algorithmic}
\end{algorithm}

\subsection{Prompt Examples}
Prompts used in our framework, especially Sokoban task, are introduced in Prompt~\ref{box:prompt_ours_state_projection}, Prompt~\ref{box:prompt_ours_sample_plans}, Prompt~\ref{box:prompt_ours_build_plan_graph}, and Prompt~\ref{box:prompt_ours_annotate_graph}. Other prompts for are introduced in 

\begin{promptbox}[label={box:prompt_ours_state_projection}]{State Projector}
Extract the current exact player location, box locations, goal locations, and box on the goals' location, and prompt from the given history.
\end{promptbox}

\begin{promptbox}[label={box:prompt_ours_sample_plans}]{Sample Plans}

Interact with Sokoban environment to solve a task (placing every box onto goals.)

\{\{sokoban\_rule\}\}

\vspace{1em}

\#\# Instruction

\vspace{1em}

You are the player in a Sokoban environment, and your goal is to place every box on a goal within a limited number of actions (within step remaining). That means you need to make the same number of boxes on goals by placing all boxes onto goals. If a box is on a goal, it is considered satisfied.

\vspace{1em}

At each turn, you will receive the current observation.
Observation format (all coordinates are (x, y)):
- wall location: (x1, y1), (x2, y2), ...
- player location: (x, y)
- box location: (x3, y3), ...
- goal location: (x4, y4), ...
- box on goal location: (x5, y5), ...
- Step remaining: <steps\_remaining>
The "box location" list includes only boxes not on goals. The "goal location" list includes all goals.

\vspace{1em}

You are in planning mode. Using the Sokoban rules, plan the full solution (overall path) so that all boxes reach goals within the remaining steps.
Generate \{\{num\_plans\}\} diverse and valid plans.
For each plan, you must **precise** reason based on the location of player, box, goal, and the wall, and the Sokoban rules (task + mechanics). Then, generate the full solution (overall path) so that all boxes reach goals within the remaining steps, using the Sokoban rules (task + mechanics).
Based on this plan, you generate the full sequence of actions.
Generate as **diverse** as possible while maintaining success within given remaining steps.
Each plan MUST have a full action sequence (at most remaining steps of actions).
Provide exactly \{\{num\_plans\}\} plans labeled "Plan 1:" ,..., "Plan \{\{num\_plans\}\}:".

\vspace{1em}

IMPORTANT: Use the action-coordinate rules exactly. Actions update the player location: U (moving up) moves the player (x, y) -> (x, y+1), D (moving down) moves the player (x, y) -> (x, y-1), R (moving right) moves the player (x, y) -> (x+1, y), L (moving left) moves the player (x, y) -> (x-1, y). Always align your verbal directions (up/down/left/right) with these coordinate changes. When describing relative positions, use the coordinate conventions: "above" means larger y, "below" means smaller y, "right" means larger x, "left" means smaller x. Validate each step in the plan by explicitly computing the next (x, y) from the action; if the destination cell is a box, check the cell behind it and treat the move as a valid push if that cell is empty or a goal, otherwise the move is invalid and must not appear in the action sequence.
If the player pushes a box, the box will move one cell in the same direction as the player unless blocked by a wall or another box.
After any move (including a push), the player and any box must occupy different cells; they can never share the same location.
\end{promptbox}

\begin{promptbox}[label={box:prompt_ours_build_plan_graph}]{Build Plan Graph}
Simulate Sokoban plans, produce per-plan step sequences and generate the graph.

\vspace{1em}

\{\{sokoban\_rule\}\}

\vspace{1em}

IMPORTANT: Use the action-coordinate rules exactly. Actions update the player location: U (moving up) moves the player (x, y) -> (x, y+1), D (moving down) moves the player (x, y) -> (x, y-1), R (moving right) moves the player (x, y) -> (x+1, y), L (moving left) moves the player (x, y) -> (x-1, y). Always align your verbal directions (up/down/left/right) with these coordinate changes. When describing relative positions, use the coordinate conventions: "above" means larger y, "below" means smaller y, "right" means larger x, "left" means smaller x. Validate each step in the plan by explicitly computing the next (x, y) from the action; if the destination cell is a box, check the cell behind it and treat the move as a valid push if that cell is empty or a goal, otherwise the move is invalid and must not appear in the action sequence.
If the player pushes a box, the box will move one cell in the same direction as the player unless blocked by a wall or another box.
After any move (including a push), the player and any box must occupy different cells; they can never share the same location.

\vspace{1em}

\#\# Instructions

\vspace{1em}

- Simulate each plan step-by-step using the Sokoban rules and the observation.
- Build each plan's step sequence as alternating entries: node -> action -> node -> action -> ...
- A node entry must include a node id and the predicted observation text.
- Node ids must be unique across all plans (no duplicate node\_id between plans).
- IMPORTANT: Observations must include ONLY player location and box location (no walls, no goals, no box-on-goal).
- The observation text must start with a short "Thought: ..." line that explains how the previous action moves player/box while considering walls.
- An action entry must include only the action (U/D/L/R). No separate thought key in action entries.
- Preserve ALL actions from every plan (do not drop steps).

\vspace{1em}

Return JSON only:
\{
  "plan\_sequences": [
    \{
      "plan\_id": "plan1",
      "steps": [
        \{
          "kind": "node",
          "node\_id": "nodeX",
          "thought": "Initial states"
          "observation": "player location: (x, y)\\n box location: (x1, y1), ..."
        \},
        \{
          "kind": "action",
          "action": "U"
        \},
        \{
          "kind": "node",
          "node\_id": "nodeY",
          "thought": "After moving up (U), player is at (x, y+1) unless blocked; box moves to (x1, y1+1) if pushed."
          "observation": "player location: (x, y+1)\\n box location: (x1, y1+1), ..."
        \}
      ]
    \}
  ]
\},

\vspace{1em}

- After all plan sequences are built, merge identical nodes when the observation text matches exactly.
- Preserve ALL actions from every plan (do not drop steps).
- Construct the full graph (nodes + edges) from the merged nodes.
- Generate a thought for each edge based on the transition between observations.
- Mark goal nodes explicitly with "is\_goal": true in full\_graph nodes.
- JSON keys must appear in this order: reasoning, merge\_log, full\_graph.
- In merge\_log entries, put "reason" first, then "kept\_node", then "merged\_nodes".

\vspace{1em}

\{
  "reasoning": "overall reasoning for node merging and graph construction",
  "merge\_log": [
    \{
      "reason": "same observation text",
      "kept\_node": "nodeX",
      "merged\_nodes": ["nodeXX", ...]
    \}
  ],
  "full\_graph": \{
    "nodes": [
      \{
        "id": "nodeX",
        "observation": "player location: (x, y)\\nbox location: (x1, y1), ...",
        "is\_goal": true or false
      \}
    ],
    "edges": [
      \{
        "from": "nodeX",
        "to": "nodeY",
        "thought": "why this step is taken",
        "action": "<U/D/L/R>"
      \}
    ]
  \}
\}
\end{promptbox}

\begin{promptbox}[label={box:prompt_ours_annotate_graph}]{Annotate Graph}

You score all the states in the graph. Higher score means closer to solving.
Goal: move all boxes onto goals (use goal locations from the observation).
Assign score 1.0 to goal states. Assign score -1.0 if, starting from this state, there is no valid sequence of actions that can ever reach any goal state (deadlocked/unreachable). All other states should have score 0.0.
Consider wall locations, blocked pushes, and required pushing routes when judging reachability.
Only use scores -1.0, 0.0, or 1.0 (no other values).
For each state, provide a short reasoning sentence before assigning its score.
Return JSON only with both reasons and scores, for example:
\{
  "reasons": \{
    "s0": "reasoning for s0: whether any sequence can reach a goal, considering walls, blocked pushes, and deadlock signals (e.g., box stuck in a corner with no goal)",
    ...
  \},g
  "scores": \{
    "s0": 0.0,
    ...
  \}
\}.

\end{promptbox}

\newpage

\section{Additional Experiment}

\subsection{Planning \& Sampling Error Estimation}
\label{app:error_estimation}
This section describes how we extract the intended action from \texttt{Thought} and how we compute planning and sampling error rates reported in \autoref{tab:planning_sampling_errors}. Also, we do the qualitative analysis of each errors.

\paragraph{Setting.} For each step, we extract the intended next action $\hat a_t \in \{U,D,L,R\}$ from the agent's \texttt{Thought} using \texttt{gpt-4.1-mini} as a parser. 

Then, given the current state $s_t$ and the intended action $\hat a_t$, we simulate a one-step transition to obtain $s'_{t+1} = f(s_t,\hat a_t)$.
We run a Breadth-First Search (BFS) oracle from $s'_{t+1}$ to compute the minimum remaining steps to reach the goal, denoted $b(s'_{t+1})$.
Let $B(t+1)$ denote the remaining step budget at time $t+1$.
We mark $\hat a_t$ as non-viable if $d(s'_{t+1}) > B(t+1)$ or if $b(s'_{t+1}) = \infty$, and we count it as a planning error. 

Lastly, let $a_t$ denote the executed action from \texttt{Action}.
We define the sampling error indicator at step $t$ as $\mathbf{1}_{[a_t \neq \hat a_t]}$ and compute the sampling error rate by averaging this indicator over steps.

\paragraph{Sampling Error Example.}
In our analysis, we observe that sampling errors often manifest as a mismatch between the action implied by the LM’s plan and the action actually executed.
\autoref{tab:sampling_error_example} shows a representative example: the \emph{Thoughts} column contains the LM-generated planning trace, which clearly implies the next action $R$ (pushing the box right to the goal), while the executed action deviates to $U$.
Comparing the planned action inferred from the LM’s reasoning with the executed action reveals such stochastic deviations at execution time, which we refer to as \emph{sampling error}.

\begin{table}[t]
\centering
\small
\begin{tabular}{p{0.68\linewidth} c c}
\toprule
Thoughts & Planned & Executed \\
\midrule
The player is at (6, 4), right below the box at (6, 5). The player can push the box right from (6, 5) to (7, 5), which is the goal. This push will complete the task. The immediate next action is to \textcolor{blue}{push the box right} by moving right from (6, 4) to (7, 4), pushing the box from (6, 5) to (7, 5). & \textcolor{blue}{R} & \textcolor{red}{U} \\
\bottomrule
\end{tabular}
\caption{\textbf{Example of sampling error.} Although the plan specifies the immediate next action as $R$, the executed action is $U$.}
\label{tab:sampling_error_example}
\end{table}

\subsection{Additional Empirical Results}\label{app:additional_experiment}

\begin{table}[t]
\centering
\small
\caption{
\textbf{Comparing \ours with best-of-$N$ integration on the ReAct and PA frameworks~\citep{BestOfN}.}
Results are on \textbf{Sokoban (easy)}, and we report \textbf{mean $\pm$ standard error}.
For a compute-matched comparison, the Best-of-$N$ variants sample the same number of plans per step ($M$) as \ours.
Even under this matched sampling budget, \ours consistently achieves higher success rates, suggesting that the gains are not solely due to increased sampling.
}
\begin{tabular}{l c}
\toprule
Method & Success Rate (\%) \\
\midrule
\ours (M=4) & $\mathbf{46.0 \pm 5.0}$ \\
Plan-and-Act & $20.0 \pm 4.0$ \\
Plan-and-Act (M=4) & $\underline{22.0 \pm 4.1}$ \\
ReAct & $4.0 \pm 2.0$ \\
ReAct (M=4) & $8.0 \pm 2.7$ \\
\bottomrule
\end{tabular}
\label{tab:best_of_n}
\end{table}

\paragraph{Comparison with Best-of-$N$ Methods.}
We compare \ours against compute-matched best-of-$N$ variants of ReAct and PA~\citep{BestOfN, AgentDistill}, as shown in \autoref{tab:best_of_n}.
At each step, these baselines sample the same number of plans ($M$) as \ours and select the best outcome, yielding a roughly comparable trajectory-sampling budget (and hence similar inference-token usage).
On Sokoban (easy) with $M{=}4$, \ours achieves $46.0\%$ success, outperforming PA ($20.0\%$) and PA-best-of-$4$ ($22.0\%$), as well as ReAct ($4.0\%$) and ReAct-best-of-$4$ ($8.0\%$).
These results indicate that \ours's gains are not solely attributable to increased sampling, but rather to improved plan selection and execution under feasibility constraints.

\paragraph{Success under Larger Step Budgets.}
In \autoref{fig:success_larger_budget}, as we increase the step budget (i.e., provide more slack beyond the oracle minimum steps), we observe that ReAct and Plan-and-Act largely plateau, exhibiting only marginal changes in success despite the additional budget. In contrast, \ours improves monotonically from 46\% to 75\% success as the normalized step budget $B/S_{\min}$ increases, indicating that it can effectively exploit extra steps for recovery rather than getting stuck in irrecoverable failures. This trend is consistent with our error decomposition: \ours reduces planning errors that transition the agent into dead-end (non-viable) states (as measured by our Sokoban dead-end oracle) and mitigates sampling-induced execution deviations via constrained decoding and mismatch-triggered replanning. Consequently, \ours lowers the probability of entering irrecoverable dead-ends early, allowing additional budget to translate into sustained gains in task success.

\begin{figure}[t]
    \centering
    \resizebox{1.0\linewidth}{!}{
        \begin{subfigure}[b]{0.49\linewidth}
            \centering
            \includegraphics[width=\linewidth]{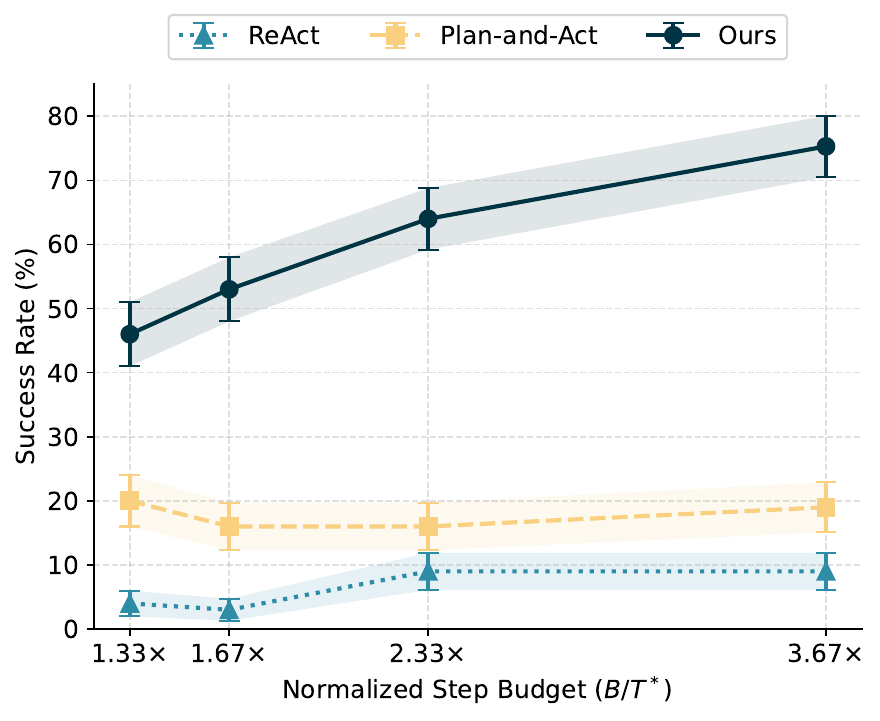}
            \caption{Success under Larger Step Budgets}
            \label{fig:success_larger_budget}
        \end{subfigure}
        \hfill %
        \begin{subfigure}[b]{0.49\linewidth}
            \centering
            \includegraphics[width=\linewidth]{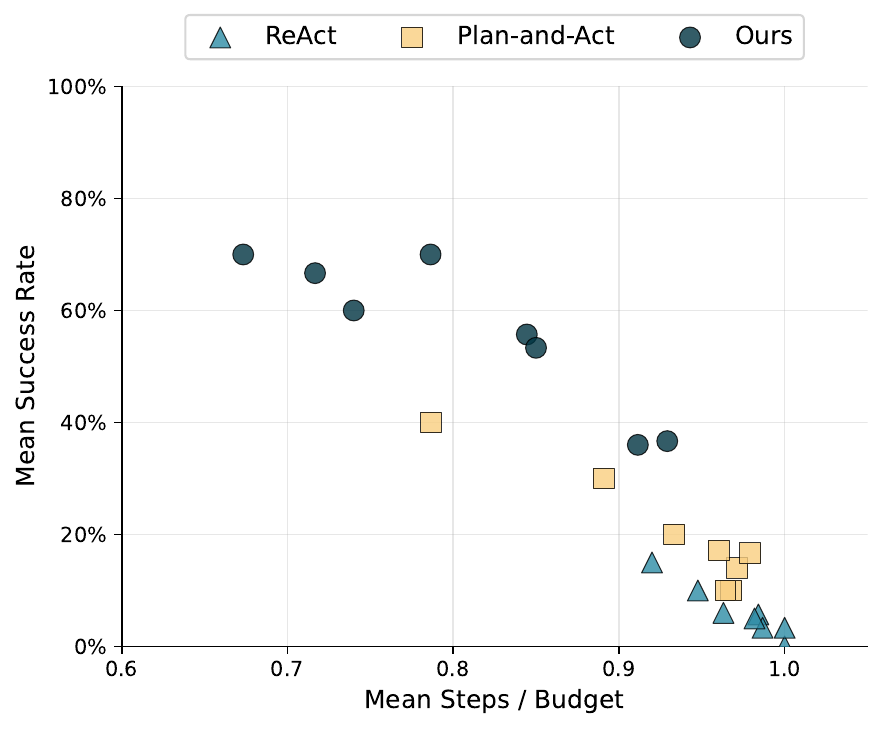}
            \caption{Success–Cost Trade-off}
            \label{fig:success_efficiency}
        \end{subfigure}
    }
\caption{
\textbf{Impact of larger step budgets and success-cost trade-off on Sokoban.}
\textbf{(a)} Success rate as a function of the normalized step budget $B/S_{\min}$, where $S_{\min}$ denotes the oracle minimum number of steps (here $S_{\min}=6$ and $B\in\{8,10,14,22\}$). Error bars indicate $\pm$ standard error across runs. Our framework is more higher success rate as step budget increases.
\textbf{(b)} Success rate versus step consumption (steps used divided by the budget). Methods closer to the top-left achieve higher success while consuming fewer steps per budget, indicating a better success--cost trade-off. \ours can be both efficient and powerful compared to other frameworks.
}
    \label{fig:appendix-more-analysis}
    \vspace{-10pt}
\end{figure}

\paragraph{Success–Cost Trade-off.} In \autoref{fig:success_efficiency}, the x-axis measures step consumption (steps used divided by the budget), so better methods lie toward the top-left (higher success with lower cost consumption). 
we find that \ours is closer to this region, achieving higher success while using fewer steps per budget on average compared to ReAct and PA. This result indicates that \ours does not only improve the success rate but also enhances the efficiency.

\subsection{Qualitative Analysis}

\paragraph{Graph Example.}
\begin{figure*}[t!]
    \centering
    \includegraphics[width=\linewidth]{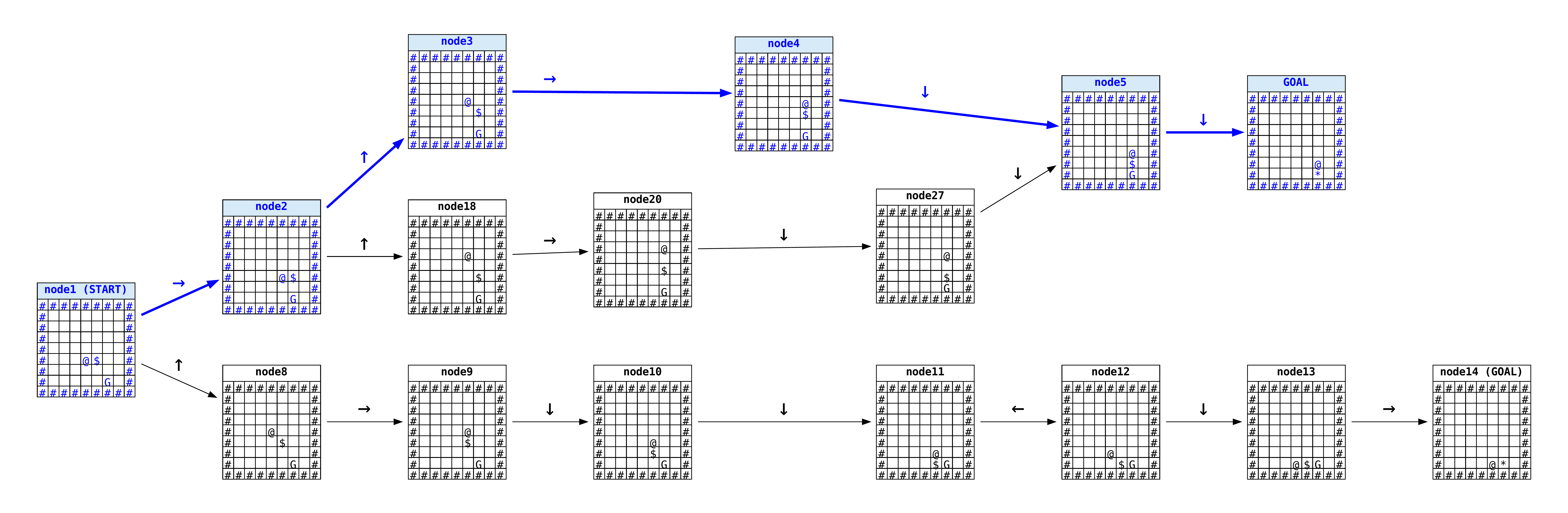}
\vspace{-15pt}
\caption{\textbf{Graph construction and plan selection example in Sokoban}. Graph is constructed by multiple paths from LLM lookahead plans. Then, a \textcolor{blue}{blue path} is selected by a formal solver (ILP). When performing constrained execution, \ours can accurately performs its plan, achieving the goal.}
    \label{fig:graph_example_sokoban}
\vspace{-15pt}
\end{figure*}

Figure \ref{fig:graph_example_sokoban} illustrates an example of the graph constructed by \ours for Sokoban. Nodes represent Sokoban states, where the goal is denoted as $G$, boxes as $\$$, the player as @, and walls as $\#$. Edges represent the actions. Given that 5 steps remain, the blue path shown in Figure \ref{fig:graph_example_sokoban} is the one selected by the external solver.
This example demonstrates that our framework is capable of finding a feasible path by using formal solver when the graph is well-constructed.

\begin{figure}[t]
    \centering
    \begin{subfigure}[b]{1.0\linewidth} %
        \centering
        \includegraphics[width=\linewidth]{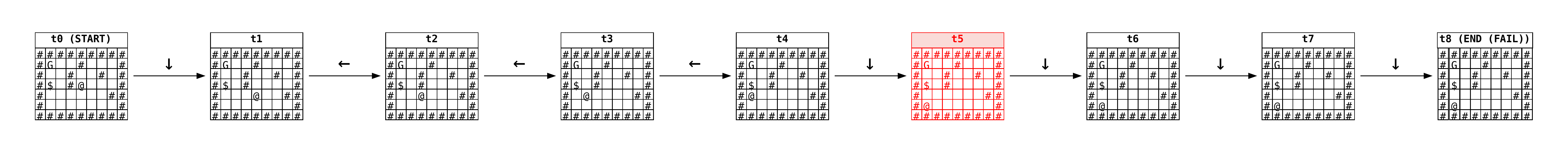}
        \caption{ReAct}
        \label{fig:react_example}
    \end{subfigure}
    
    \vspace{10pt} %
    
    \begin{subfigure}[b]{1.0\linewidth}
        \centering
        \includegraphics[width=\linewidth]{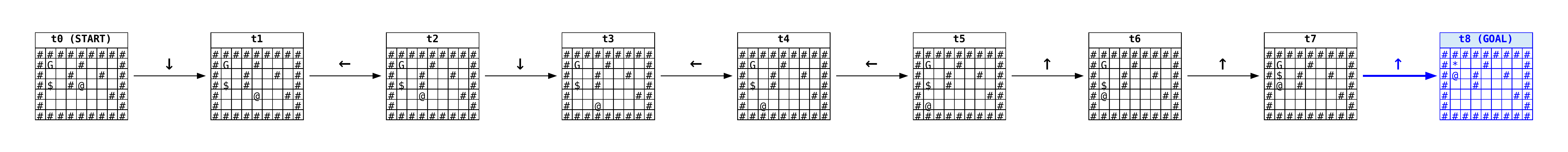}
        \caption{Our Framework}
        \label{fig:ours_example}
    \end{subfigure}

\caption{\textbf{Qualitative comparison between ReAct and TAPE.}
We visualize representative execution trajectories under the same action budget ($B=8$).
\textbf{(a) ReAct} makes a mistake at $t{=}5$ (red), after which it repeatedly deviates and fails to reach the goal within the budget, terminating at $t{=}8$ (END/FAIL).
\textbf{(b) TAPE} selects a feasible path and executes it reliably, reaching the goal within the same budget at $t{=}8$ (blue).}
    \label{fig:appendix-qualitative-comparison}
\end{figure}

\paragraph{Qualitative Comparison.}
\autoref{fig:appendix-qualitative-comparison} provides a representative trajectory-level comparison between ReAct and TAPE under the same action budget.
ReAct makes an early mistake (at $t{=}5$), after which the subsequent actions continue to deviate and the agent fails to reach the goal within the budget, terminating in failure.
In contrast, TAPE reaches the goal within the same budget by selecting a feasible plan and executing it more reliably, avoiding irrecoverable states.
Overall, this qualitative evidence supports our quantitative findings in \autoref{sec:experiment} that TAPE mitigates both planning and

\end{document}